\documentclass[twoside]{article}
\usepackage[accepted]{aistats2017}

\usepackage{hyperref,color}
\definecolor{darkgreen}{rgb}{0,.4,.2}
\definecolor{darkblue}{rgb}{.1,.2,.6}
\definecolor{brightblue}{rgb}{0,0.6,0.8}
\hypersetup{
  colorlinks=true,
  linkcolor=darkblue,
  citecolor=darkblue,
  filecolor=darkgreen,
  urlcolor=darkgreen
}
\usepackage[T1]{fontenc}    %

\usepackage{url}            %
\usepackage{booktabs}       %
\usepackage{amsfonts}       %
\usepackage{nicefrac}       %
\usepackage{microtype}      %
\usepackage{array}

\usepackage{times}
\usepackage{graphicx}
\usepackage{mathtools}

\usepackage{algorithm}
\usepackage{algorithmic}

\usepackage{amssymb}
\usepackage{multirow}
\usepackage{amsmath}
\usepackage{amsthm}
\usepackage{mathrsfs}
\usepackage{graphicx}
\usepackage{caption, subcaption}
\usepackage{xspace}
\usepackage{float}
\usepackage{wrapfig}
\usepackage{enumitem}
\setitemize{noitemsep,topsep=0pt,parsep=2pt,partopsep=0pt}

\usepackage{natbib}
\renewcommand{\cite}{\citep}

\def\R{{\mathbb{R}}}

\def\to{{\,\rightarrow\,}}

\mathchardef\mhyphen="2D

\newcommand{\norm}[1]{{ \left\lVert#1\right\rVert }}

\newcommand{\vertiii}[1]{{\left\vert\kern-0.25ex\left\vert\kern-0.25ex\left\vert #1
    \right\vert\kern-0.25ex\right\vert\kern-0.25ex\right\vert}}

\newcommand{\vect}[1]{{\boldsymbol{#1}}}

\def\vd{{\vect{d}}}

\def\vx{{\vect{x}}}

\def\bb{{\mathbf{b}}}

\def\bd{{\mathbf{d}}}
\def\be{{\mathbf{e}}}

\def\bp{{\mathbf{p}}}

\def\br{{\mathbf{r}}}
\def\bs{{\mathbf{s}}}

\def\bv{{\mathbf{v}}}

\def\bx{{\mathbf{x}}}
\def\by{{\mathbf{y}}}
\def\bz{{\mathbf{z}}}

\def\bG{{\mathbf{G}}}

\def\bM{{\mathbf{M}}}

\def\0{{\mathbf{0}}}

\def\bbR{{\mathbb{R}}}

\def\cA{\mathcal{A}}
\def\cB{\mathcal{B}}

\def\cD{\mathcal{D}}

\def\cH{\mathcal{H}}
\def\cI{\mathcal{I}}
\def\cJ{\mathcal{J}}

\def\cP{\mathcal{P}}
\def\cQ{\mathcal{Q}}

\def\cS{\mathcal{S}}

\newtheorem*{rep@theorem}{\rep@title}
\newcommand{\newreptheorem}[2]{%
\newenvironment{rep#1}[1]{%
 \def\rep@title{#2 \ref{##1}}%
 \begin{rep@theorem}}%
 {\end{rep@theorem}}}
\newreptheorem{lemma}{Lemma'}
\newreptheorem{definition}{Definition'}
\newreptheorem{proposition}{Proposition'}
\newreptheorem{theorem}{Theorem'}

\newtheorem{theorem}{Theorem}
\newtheorem{definition}[theorem]{Definition}
\newtheorem{corollary}[theorem]{Corollary}
\newtheorem{lemma}[theorem]{Lemma}

\renewcommand{\text}[1]{{\textnormal{#1}}}

\DeclareMathOperator*{\argmin}{arg\,min}
\DeclareMathOperator*{\argmax}{arg\,max}
\DeclareMathOperator{\conv}{conv}
\DeclareMathOperator{\lin}{lin}
\DeclareMathOperator{\diam}{\mathrm{diam}}
\DeclareMathOperator{\radius}{\mathrm{radius}}
\DeclareMathOperator{\clip}{clip}

\newcommand{\lmo}{\textsc{LMO}\xspace}
\DeclareMathOperator{\inr}{\mathrm{inr}}
\DeclareMathOperator{\mdw}{mDW(\cA)}

\newcommand{\MP}{{\textsf{\tiny MP}}}
\newcommand{\FW}{{\textsf{\tiny FW}}}
\newcommand{\FWa}{{\textsf{\tiny FW}_\alpha}}

\newcommand{\domain}{\cQ}
\newcommand{\Cf}{C_{f,\cA}} 
\newcommand{\CfMP}{C_{f,\cA}^\MP}
\newcommand{\CfMPr}{C_{f,\rho \cA}^\MP}

\newcommand{\CfAW}{C_{f,\cA}^{\textsf{\tiny AW}}}
\newcommand{\muf}{\mu_{f, \cA}^\MP}
\newcommand{\mufr}{\mu_{f, \rho \cA}^\MP}

\newcommand{\deltaFW}{\delta_\FW}
\newcommand{\deltaMP}{\delta_\MP}

\begin{document}

\runningtitle{A Unified Optimization View on Generalized Matching Pursuit and Frank-Wolfe}

\twocolumn[

\aistatstitle{A Unified Optimization View on\\ Generalized Matching Pursuit and Frank-Wolfe}

\aistatsauthor{ Francesco Locatello \And Rajiv Khanna{\normalfont*} \And Michael Tschannen{\normalfont*} \And Martin Jaggi }

\aistatsaddress{ ETH Z\"urich 
\And UT Austin \And ETH Z\"urich \And EPFL } 
]

\begin{abstract}\vspace{-2mm}
Two of the most fundamental prototypes of greedy optimization are the matching pursuit and Frank-Wolfe algorithms. In this paper, we take a unified view on both classes of methods, leading to the first explicit convergence rates of matching pursuit methods in an optimization sense, for general sets of atoms. We derive sublinear ($1/t$) convergence for both classes on general smooth objectives, and linear convergence on strongly convex objectives, as well as a clear correspondence of algorithm variants. Our presented algorithms and rates are affine invariant, and do not need any incoherence or sparsity assumptions.
\vspace{-2mm}
\end{abstract}

\section{Introduction}
\vspace{-2mm}

During the past decade, greedy algorithms have attracted significant attention and led to many success stories in machine learning and signal processing (e.g., compressed sensing), and  optimization in general.
The most prominent representatives are matching pursuit (MP) algorithms on one hand \cite{Mallat:1993gu}, such as, e.g., orthogonal matching pursuit (OMP) \cite{chen1989orthogonal, Tropp:2004gc}, and on the other hand Frank-Wolfe (FW)-type algorithms \cite{Frank:1956bt}.
Both operate in the setting of minimizing an objective over (combinations of) a given set of atoms, or dictionary elements.

The two classes of methods have very strong similarities, in the sense that they in each iteration rely on the very same subroutine, namely selecting the atom of largest inner product with the negative gradient, i.e., what we call the \emph{linear minimization oracle} (\lmo).
Yet, the main difference is that MP methods optimize over the linear span of the atoms, while FW methods optimize over their convex hull.

Despite the vast literature on MP-type methods which typically gives recovery guarantees for sparse signals, surprisingly little is known about MP algorithms in terms of optimization, i.e., how many iterations are needed to reach a defined target accuracy. In particular, we are not aware of any general-purpose explicit convergence rates, which hold for an arbitrary given set of atoms (here ``explicit'' means that the result must not depend on iteration-dependent quantities). 
Indeed, in the context of sparse recovery, convergence rates typically come as a byproduct of the recovery guarantees and hence depend on very strong assumptions (from an optimization perspective), such as incoherence or restricted isometry properties of the atom set \cite{Tropp:2004gc,davenport2010analysis}. Motivated by this line of work, \cite{Gribonval:2006ch,Temlyakov:2013wf, Temlyakov:2014eb, nguyen2014greedy} specifically target convergence rates but still rely on incoherence properties. 
On the other hand, FW methods are well understood from an optimization perspective, with strong explicit convergence results available for a large class of input problems, see, e.g., \cite{Jaggi:2013wg,LacosteJulien:2015wj} for a recent account.

In this paper, we provide a unified view on MP and FW algorithms from an optimization perspective.
Our joint understanding of both classes of algorithms has several benefits:
\begin{itemize}
\item We provide a clear presentation of MP methods with their FW analogues in a unified context, for the task of general convex optimization over any set of atoms from a Hilbert space. Our view also includes weight-corrective variants of MP and FW which we are able to set in direct correspondence.
\item Our derived convergence rates (sub-linear for the case of smooth objective, and linear/geometric for the case of smooth and strongly convex objective) are the first explicit optimization rates for MP methods, for general atom sets, to the best of our knowledge. We set the new rates and their complexity constants in context with existing FW rates. Our linear convergence rate of MP is expressed in terms of a new quantity called the \emph{minimal intrinsic directional width} of the atoms.
\item We allow for approximate subroutines in all proposed MP and FW variants, that is the use of an approximate linear oracle (LMO). The level of approximation quality is reflected in all convergence rates.
\item Additionally, we give affine invariant extensions of the MP and FW algorithm variants, as well as convergence rates in terms of affine invariant quantities. That is, the algorithms and rates will be invariant under affine transformations and re-parameterizations of the optimization domain (a property which was known for Newton's method and FW methods, but is novel in the MP context).
\end{itemize}
\vspace{-2mm}
\paragraph{Motivation.}
The setting of optimization over linear or convex combinations of atoms has served as a very useful template in many applications, since the choice of the atom set conveniently allows to encode structure desired for the use case. Apart from many applications based on sparse vectors, the use of rank-1 atoms gives rise to structured matrix and tensor factorizations, see, e.g., \cite{wang2014matrixcompletion,Yang:2015wy, yaogreedy,guo2017efficient}. For example, minimizing the Bregman Divergence over a set of structured rank-1 matrices yields an exponential family structured PCA \citep{gunasekar2014exponentia}. Other applications include multilinear multitask learning \citep{romera2013multilinear}, matrix completion and image denoising \citep{tibshirani2015general}.
\vspace{-1mm}

\paragraph{Complexity Constants and Coherence.}
While our sub-linear convergence rates for MP and FW only depend on bounded norm of the iterates and on the diameter of the atom set, the linear rates also depend on our notion of minimal intrinsic directional width. 
In contrast to the notion of (cumulative) coherence commonly used in the context of MP and OMP \cite{Gribonval:2006ch}, our width complexity notion is more robust, e.g., w.r.t. addition of new atoms, and leading to provably better bounds than coherence.
Furthermore, our linear rates are significantly easier to interpret than the linear rates obtained for FW algorithm variants in \cite{LacosteJulien:2015wj} which rely on a complex geometric quantity called pyramidal width.
Finally, we elucidate the relationship between FW algorithms and our proposed generalized MP variants, by showing that the iterates of FW converge to those of MP as $O(1/\alpha)$, if the atom set of FW is scaled by a growing factor $\alpha$.

We note that a few recent works \cite{ShalevShwartz:2010wq, Temlyakov:2013wf, Temlyakov:2014eb, Temlyakov:2012vg, nguyen2014greedy,yaogreedy} proposed similar algorithms extending MPs to general smooth objective functions, although with less general convergence
rates and without studying the algorithms in the larger context of MP and FW. The relation to these works is discussed in detail in Section~\ref{sec:priorwork}. 
\vspace{-1mm}

\paragraph{Notation.} 
Let $[d]$ be the set $\{1,2,\ldots,d\}$. 
Given a non-empty subset $\cA$ of some vector space, let $\conv(\cA)$ be the convex hull of the set~$\cA$, and let $\lin(\cA)$ denote the linear span of the elements in~$\cA$. 
Given a closed set $\cA$ we call its diameter $\diam(\cA)=\max_{\bz_1,\bz_2\in\cA}\|\bz_1-\bz_2\|$ and its radius $\radius(\cA) = \max_{\bz\in\cA}\|\bz\|$. Note that for convex hulls of finite atom sets $\cA$ we have $\diam(\conv(\cA)) = \diam(\cA)$, i.e., the diameter is attained at two vertices~\cite{Ziegler:1995td}. 
$\|\bx\|_\cA := \inf \{ c > 0 \colon \bx \in c \cdot \conv (\cA) \}$ is the atomic norm of $\bx$ over a set $\cA$ (also known as the gauge function of $\conv (\cA)$). We call a subset $\cA$ of a Hilbert space symmetric if $\cA = -\cA$. 
We write $\clip_{[0,1]}(s) := \max\{0,\min\{1,s\}\}$.
\vspace{-2mm}

\section{Matching Pursuit and Frank-Wolfe} \label{sec:revMPFW}
\vspace{-2mm}

We start by reviewing the MP \cite{Mallat:1993gu}, the OMP \cite{chen1989orthogonal, Tropp:2004gc}, and the FW algorithm \cite{Frank:1956bt, Jaggi:2013wg} in Hilbert spaces. The setting considered throughout this paper is the following. Let $\cH$ be a Hilbert space with associated inner product $\langle \bx, \by\rangle, \,\forall \, \bx,\by \in \cH$. The inner product induces the norm $\| \bx \|^2 := \langle \bx, \bx \rangle,$ $\forall \, \bx \in \cH$. Let $\cA \subset \cH$ be a non-empty bounded set (the set of atoms or dictionary) and let $f \colon \cH \to \bbR$ be convex and $L$-smooth ($L$-Lipschitz gradient in the finite-dimensional case). If $\cH$ is an infinite-dimensional Hilbert space, then $f$ is assumed to be \emph{Fr{\'e}chet differentiable}.

In each iteration, both the MP/OMP and the FW algorithm query a so-called linear minimization oracle (\lmo) which solves the optimization problem
\begin{equation}\label{eq:FWLMO}
\lmo_\cD(\by) := \argmin_{\bz\in\cD} \,\langle \by, \bz \rangle
\end{equation} 
for given $\by\in\cH$ and $\cD \subset \cH$. As computing an exact solution \eqref{eq:FWLMO}, depending on $\cD$, is often hard in practice, it is desirable to rely on an \emph{approximate} \lmo that returns an approximate minimizer of \eqref{eq:FWLMO}. Different notions of approximate \lmo\!s are discussed in more detail in Section~\ref{sec:approxlmo}.

MP and OMP, presented in Algorithm \ref{algo:OMP}, aim at approximating a target point $\by \in \cH$ as well as possible in the least-squares sense using no more than $T$ atoms form a possibly countable or finite dictionary $\cA \subset \cH$.

\begin{algorithm}[h]
  \caption{(Orthogonal) Matching Pursuit}
  \label{algo:OMP}
\begin{algorithmic}[1]
  \STATE \textbf{init} $\bx_{0} \in \lin(\cA)$ $\cS = \left\lbrace\bx_0 \right\rbrace $ 
  \STATE \textbf{for} {$t=0\dots T$}
  \STATE \qquad Find $\bz_t := (\text{Approx-}) \lmo_{\cA \cup -\cA}(-\by+\bx_t)$
  \STATE \qquad $\cS = \cS \cup \bz_t$
  \STATE \qquad Update MP: $\bx_{t+1}: = \underset{\substack{\bx := \bx_{t} + \gamma \bz_t \\ \gamma \in \R}}{\argmin}  \|\by-\bx\|^2$, or
  \STATE \qquad Update OMP: $\bx_{t+1}: = \argmin_{\bx\in\lin(\cS)}  \|\by-\bx\|^2$
  \STATE \textbf{end for}
\end{algorithmic}
\end{algorithm}

At each iteration, OMP adds a new atom to the active set~$\cS$ and computes the new iterate as the least-squares approximation of $\by$  in terms of the atoms in $\cS$. As a result, the residual $\br_{t+1} := \by - \bx_{t+1}$ is orthogonal to $\lin(\cS)$. This is in contrast to MP, which only minimizes the residual error $\| \br_{t+1} \|^2$ w.r.t.~$\bz_t$ so that $\br_{t+1}$ is orthogonal to $\bz_t$, but not necessarily to all~$\bz_{t'}$, $t' \leq t-1$. Note that MP does not require to maintain the active set $\cS$ as the update only relies on $\bz_t$. Also note that in the signal processing literature MP and OMP are typically formulated using $\bz_t := \argmax_{\bz\in\cA} \,|\langle \by-\bx_t, \bz \rangle|$ in Line~3 of Algorithm \ref{algo:OMP} instead of $\bz_t :=\lmo_{\cA \cup -\cA}(-\by+\bx_t)$. The solution of this alternative \lmo definition is equal to that of $\lmo_{\cA \cup -\cA}$ up to the sign, so that the iterates $\bx_t$ are identical for both definitions. 
Relying on $\lmo_{\cA \cup -\cA}$ here allows to better illustrate the parallels between MP/OMP and FW.

We now turn to the FW algorithm \cite{Frank:1956bt,Jaggi:2013wg}, also referred to as conditional gradient in the literature. The FW algorithm, presented in Algorithm \ref{algo:FW}, targets the optimization problem
\begin{align}\label{eq:FWproblem}
\min_{\bx\in\cD} f(\bx),
\end{align}
where $\cD \subset \cH$ is convex and bounded. In many applications,~$\cD$ is the convex hull of a dictionary $\cA$, i.e., $\cD=\conv(\cA)$, in which case $\lmo_\cD(\by) = \lmo_\cA(\by)$.

\begin{algorithm}[h]
  \caption{Frank-Wolfe}
  \label{algo:FW}
\begin{algorithmic}[1]
  \STATE \textbf{init} $\bx_{0} \in \conv(\cA)$
  \STATE \textbf{for} {$t=0\dots T$}
  \STATE \qquad Find $\bz_t := (\text{Approx-}) \lmo_{\cA}(\nabla f(\bx_{t}))$
  \STATE \qquad \emph{Variant 0:} $\gamma := \frac{2}{t+2}$
  \STATE \qquad \emph{Variant 1:} $\gamma := \displaystyle\argmin_{\gamma\in[0,1]}  f({\bx_{t} \!+\! \gamma(\bz_t-\bx_{t})})$\vspace{-1mm}
  \STATE \qquad \emph{Variant 2:} $\gamma := \clip_{[0,1]}\!\big[ \frac{\langle -\nabla f(\bx_t), \bz_t - \bx_t\rangle}{\diam_{\|.\|\!}(\cA)^2 L} \big]$
  \STATE \qquad \emph{Variant 3:} $\gamma := \clip_{[0,1]}\!\big[ \frac{\langle -\nabla f(\bx_t), \bz_t - \bx_t\rangle}{\|\bz_t - \bx_t\|^2 L} \big]$
  \STATE \qquad Update $\bx_{t+1}:= \bx_{t} + \gamma(\bz_t-\bx_{t})$
  \STATE \textbf{end for}
\end{algorithmic}
\end{algorithm}

At each iteration, the FW algorithm selects a new atom~$\bz_t$ from $\cD$ by querying the LMO and computes the new iterate as a convex combination of $\bz_t$ and the old iterate $\bx_t$. As discussed in \cite{Jaggi:2013wg}, the convex update can be performed either by line search (line 5 in Algorithm \ref{algo:FW}) or as a convex combination of all previously selected atoms~$\bz_{t'}$, $t'\leq t$.

The steps in Line 3 of MP (Algorithm \ref{algo:OMP}) and Line 3 of FW (Algorithm \ref{algo:FW}) (finding the step direction) are identical up to symmetrization of $\cA$. This is seen as follows. Recall that MP and OMP approximate $\by$ in the least-squares sense, i.e., they aim at minimizing $f(\bx) := \frac{1}{2} \| \by -\bx \|^2$. For this choice of $f$ we have $\nabla f(\bx_t) = -\by + \vx_t = -\br_t$, i.e., $\lmo_{\cA \cup -\cA} (-\br_t) = \lmo_{\cA \cup -\cA}(\nabla f(\bx_t) )$.

\section{Greedy Algorithms in Hilbert Spaces}
\label{sec:generalgreedy}
\vspace{-2mm}

We present new greedy algorithms---inspired by MP, OMP, and FW---for the minimization of functions $f$ over a convex and bounded set $\cD \subset \cH$, or over the linear span of a dictionary $\cA \subset \cH$. As MP, OMP, and FW, these algorithms alternate between querying the \lmo defined in \eqref{eq:FWLMO} and updating the current iterate $\bx_t$. 
Common to all of our algorithms is that their update step minimizes an upper bound of $f$ at $\bx_t$, given as\vspace{-2mm}
 \begin{equation}\label{eq:QuadraticUpperBound}
 g_{\bx_{t}}(\bx) := f(\bx_{t}) + \langle\nabla f(\bx_{t}), \bx-\bx_{t}\rangle+\frac{L}{2}\|\bx-\bx_{t}\|^2
\end{equation}
where $L$ is an upper bound on the smoothness constant of~$f$ w.r.t. a chosen norm $\|.\|$. Optimizing this norm problem instead of the original $f$ objective allows for substantial efficiency gains in the case of complicated $f$ objective.

We note that our algorithms can be made \emph{affine invariant}, i.e., invariant under affine transformations and re-parameterizations of the domain, by simple modifications of the update steps. For simplicity of exposition, we present these algorithm versions, along with corresponding sub-linear and linear convergence results later in Section \ref{sec:affineInvariantAlgorithms}.

\vspace{-2mm}
\subsection{Constrained Optimization}
\vspace{-2mm}

We consider constrained optimization problems of the form~\eqref{eq:FWproblem} with $\cD := \conv(\cA)$ for some dictionary $\cA \subset \cH$. Inspired by the fully-corrective Frank-Wolfe variant (see, e.g., \cite{Holloway:1974ii,Jaggi:2013wg}) which, in each update step, re-optimizes the original objective over the convex hull of all previously selected atoms, $\conv(\cS)$, we instead propose to minimize the simpler quadratic upper bound~\eqref{eq:QuadraticUpperBound} over the atom selected at the current iteration (using line-search) or over $\conv(\cS)$. We call this algorithm variant, presented in Algorithm~\ref{algo:normCorrectiveFW}, \emph{norm-corrective Frank-Wolfe}.

\begin{algorithm}[h]
\caption{Norm-Corrective Frank-Wolfe}
\label{algo:normCorrectiveFW}
\begin{algorithmic}[1]
  \STATE \textbf{init} $\bx_{0} \in \conv(\cA)$, and $\cS:=\{\bx_{0}\}$
  \STATE \textbf{for} {$t=0\dots T$}
  \STATE \qquad Find $\bz_t := (\text{Approx-}) \lmo_{\cA}(\nabla f(\bx_{t}))$
  \STATE \qquad $\cS:=\cS\cup\{ \bz_t \}$
  \STATE \qquad Let $\bb := \bx_{t}-\frac1L \nabla f(\bx_{t})$
    \STATE \qquad Variant 0: Update $\bx_{t+1}:= \displaystyle\argmin_{\substack{\!\!\!\!\bz:=\bx_t +\gamma (\bz_t - \bx_t) \\\gamma\in [0,1]}} \!\!\!\|{\bz-\bb}\|_2^2$\\
	\qquad Variant 1: Update $\bx_{t+1}:= \displaystyle\argmin_{\bz\in \conv(\cS)} \|{\bz-\bb}\|_2^2$\\
  \STATE \qquad \emph{Optional:} Correction of some/all atoms $\bz_{0\ldots t}$
  \STATE \textbf{end for}
\end{algorithmic}
\end{algorithm}

The name ``norm-corrective'' is used to illustrate that the algorithm employs a simple squared norm surrogate function (or upper bound on $f$), which only depends on the smoothness constant $L$. 
This is in contrast to second-order optimization methods such as Newton's method, which rely on a non-uniform quadratic surrogate function at each iteration.  Importantly, we do not need to know $L$ (and the corresponding constant in the affine invariant algorithm versions in Section \ref{sec:affineInvariantAlgorithms}) exactly in any of the proposed algorithms; an upper bound is always sufficient to ensure convergence. 
Finding the closest point in norm can typically be performed much more efficiently than solving a general optimization problem, such as if we would minimize $f$ over the same domain, which is what the ``fully-corrective'' algorithm variants require in each iteration.
Approximately solving the subproblem in Variant 1 can be done efficiently using projected gradient steps on the weights (as projection onto the simplex and L1 ball is efficient). Assuming a fixed quadratic subproblem as in Variant 1, the CoGEnT algorithm of \cite{Rao:2015df} uses the same ``enhancement'' steps. The difference in the presentation here is that we address general~$f$, so that the quadratic correction subproblem changes in every iteration in our case.
\vspace{-3mm}
\subsection{Optimization over the linear span of a dictionary}
\vspace{-2mm}

We now move on to optimization over linear span of a dictionary $\cA \subset \cH$, i.e., we consider problems of the form\vspace{-1mm}
\begin{align} \label{eq:linprob}
\min_{\bx\in\lin(\cA)} f(\bx). 
\end{align}
To solve~\eqref{eq:linprob}, we present the Norm-Corrective Generalized Matching Pursuit (GMP) in Algorithm~\ref{algo:generalgreedy} which is again based on the quadratic upper bound \eqref{eq:QuadraticUpperBound} and can be seen  as an extension of MP and OMP to smooth functions $f$.

\begin{algorithm}[h]
  \caption{Norm-Corrective Generalized Matching Pursuit}
  \label{algo:generalgreedy}
\begin{algorithmic}[1]
  \STATE \textbf{init} $\bx_{0} \in \lin(\cA)$, and $\cS:=\{\bx_{0}\}$
  \STATE \textbf{for} {$t=0\dots T$}
  \STATE \qquad Find $\bz_t := (\text{Approx-}) \lmo_\cA(\nabla f(\bx_{t}))$
  \STATE \qquad $\cS:=\cS\cup\{ \bz_t \}$\
  \STATE \qquad $\text{ }$Let $ \bb := \bx_{t}-\frac1L \nabla f(\bx_{t})$ \\
 \STATE \qquad Variant 0:  Update $\bx_{t+1}:= \displaystyle\argmin_{\substack{\bz := \bx_{t} + \gamma \bz_t \\ \gamma \in \R}} \|{\bz-\bb}\|_2^2$\\
\qquad Variant 1: Update $\bx_{t+1}:= \displaystyle\argmin_{\bz\in \lin(\cS)} \|{\bz-\bb}\|_2^2$
  \STATE \qquad \emph{Optional:} Correction of some/all atoms $\bz_{0\ldots t}$
  \STATE \textbf{end for}
\end{algorithmic}
\end{algorithm}

Here, the updates in line 6 are again either over the most recently selected atom (Variant 0) or over all perviously selected atoms (Variant 1). However, the optimization is unconstrained as opposed to norm-corrective FW. Note that the update step in line 6 of Algorithm \ref{algo:generalgreedy} Variant 0 (line-search)  has the closed-form solution $\gamma = -\frac{\langle\bx_{t}-\bb,\bz_t\rangle}{\|\bz_t\|^2}$.

It is important to stress the fact that for Variant 1, at the end of iteration~$t$, $\nabla f(\vx_{t+1})$ is not always orthogonal to $\lin(\cS)$ as it is the case for OMP (see the discussion in Section \ref{sec:revMPFW}). 

This difference is rooted in the fact that the OMP residual $\br_{t+1} = \by - \bx_{t+1}$ (i.e., the gradient at iteration $t+1$, $\nabla f(\vx_{t+1})$) can be obtained by projecting the $\br_{t}$ (i.e., the gradient at iteration $t$, $\nabla f(\vx_{t})$) onto the orthogonal complement of $\hat \bz_t$, where $\hat \bz_t$ is obtained by orthogonalizing $\bz_t$ w.r.t. $\bz_{t'}$, $t' \leq t-1$. In other words, the OMP update step maintains orthogonality of the gradient w.r.t. the atoms selected in all previous iterations, which is not the case for general smooth functions $f$ due to varying curvature.
\vspace{-2mm}
\subsection{Discussion}
\vspace{-2mm}

The update step in line 6 in Algorithms \ref{algo:normCorrectiveFW} and \ref{algo:generalgreedy} is very similar to a projected gradient descent step with a step-size of $1/L$ (i.e., $\bb = \bx_{t}-\frac1L \nabla f(\bx_{t})$ is a gradient descent step with step size $1/L$ and the update step in line 6 is a projection of $\bb$). However, the crucial difference to projected gradient descent is that the projection step is only partial, i.e., the projection is only onto $\conv(\cS)$ and $\lin(\cS)$ instead of the entire constraint set $\conv(\cA)$ and $\lin(\cA)$ for Algorithms \ref{algo:normCorrectiveFW} and \ref{algo:generalgreedy}, respectively.

The total number of iterations $T$ of Algorithms~\ref{algo:normCorrectiveFW} and \ref{algo:generalgreedy} controls the trade-off between approximation quality, i.e., how close $f(\bx_T)$ is to the optimum $f(\bx^\star)$, and the ``structuredness'' of the (approximate) solution $\bx_T$. The structure is due to the fact that we only use $T$ atoms from $\cA$ and due to the structure of the atoms themselves (e.g., sparsity). A concrete example for an application of Algorithm \ref{algo:generalgreedy} that requires such a structure is low-rank matrix factorization: Choosing for $f$ a function measuring the approximation quality of a given matrix to a target matrix and rank-1 matrices with unit norm as atom set, $T$ controls the rank of the solution matrix. 

\vspace{-2mm}

\subsection{Approximate linear oracles and atom corrections} 
\label{sec:approxlmo}
\vspace{-2mm}

Recall that an exact \lmo is often very costly, in particular when applied to matrix (or tensor) factorization problems, while approximate versions can be much more efficient. We now generalize all the presented Algorithms to allow for an \emph{approximate \lmo}. Different notions of such an \lmo were already explored for the Frank-Wolfe framework in \cite{LacosteJulien:2013ue}. Here, we focus on multiplicative errors and define two different approximate \lmo\!s, one for Algorithm~\ref{algo:normCorrectiveFW} and another one for Algorithm~\ref{algo:generalgreedy}.
We discuss their relationship in Section~\ref{sec:FW=MP}.
Formally, for a given quality parameter $\deltaFW \in \left( 0,1\right]$ and for a given direction $\bd\in\cH$, the approximate \lmo for Algorithm \ref{algo:normCorrectiveFW} returns a vector $\tilde{\bz}\in\cA$ satisfying
\vspace{-1mm}
\begin{equation}\label{eq:inexactLMOFW}
\langle \bd,\tilde{\bz}-\bx_{t}\rangle \leq \deltaFW \min_{\bz \in \cA}\langle \bd,\bz-\bx_{t}\rangle.
\end{equation}
For given quality parameter $\deltaMP \in \left( 0,1\right]$ and given direction $\bd\in\cH$, the approximate \lmo for Algorithm~\ref{algo:generalgreedy} returns a vector $\tilde{\bz}\in\cA$ such that
\vspace{-1mm}
\begin{equation}\label{eq:inexactLMOMP}
 \langle \bd,\tilde{\bz}\rangle \leq \deltaMP\langle \bd,\bz\rangle, 
\end{equation} 
where $\bz =\lmo_\cA(\bd)$.
We will often refer to the quality parameter simply as $\delta$.

Further, as shown in line~7 of Algorithms \ref{algo:normCorrectiveFW} and \ref{algo:generalgreedy}, we also allow for correction of some/all atoms in the active set~$\cS$, see, e.g.,~\cite{Laue:2012wn,guo2017efficient}, to obtain a better objective cost while maintaining the same (small) number of atoms.

\vspace{-2mm}
\section{Sublinear Convergence Rates} \label{sec:sublin}
\vspace{-2mm}
In this section we present sub-linear convergence guarantees for Algorithms~\ref{algo:normCorrectiveFW} and \ref{algo:generalgreedy}. All proofs are deferred to the Appendix in the supplement.

\paragraph{Frank-Wolfe algorithm variants.}
We start with the convergence result for Algorithm~\ref{algo:normCorrectiveFW}, which targets optimization problems of the form \eqref{eq:FWproblem}.
Let $\bx^\star \in \argmin_{\bx\in \conv(\cA)} f(\bx)$ be an optimal solution of \eqref{eq:FWproblem}.

\begin{theorem}\label{thm:inexactFW}
Let $\cA \subset \cH$ be a bounded set and let $f \colon \cH \to \R$ be $L$-smooth w.r.t. a given norm $\|.\|$, over~$\conv(\cA)$.
Then, the Frank-Wolfe method (Algorithm~\ref{algo:FW}),
as well as Norm-Corrective Frank-Wolfe (Algorithm~\ref{algo:normCorrectiveFW}), converge for $t \geq 0$ as\vspace{-2mm}
\begin{equation*}
f(\bx_t) - f(\bx^\star) \leq \frac{2\left(\frac1\delta L\diam_{\|.\|\!}(\cA)^2 +\varepsilon_0\right)}{\delta t+2}
\end{equation*}
where $\varepsilon_0 := f(\bx_0) - f(\bx^\star)$ is the initial error in objective, and $\delta \in (0,1]$ is the accuracy parameter of the employed approximate \lmo  (Equation~\eqref{eq:inexactLMOFW}).
\end{theorem}

\paragraph{Matching pursuit algorithm variants.}

We now move on to Algorithm \ref{algo:generalgreedy} which solves optimization problems over a linear span, as given in \eqref{eq:linprob}. 
We again write $\bx^\star \in \argmin_{\bx\in \lin(\cA)} f(\bx)$ for an optimal solution.
Our rates will crucially depend on a (possibly loose) upper bound on the atomic norm of the solution and iterates: Let $\rho>0$ s.t.
\begin{equation}\label{eq:rho}
\rho \ge \max\left\lbrace \|\bx^\star\|_{\cA}, \|\bx_{0}\|_{\cA}\ldots,\|\bx_T\|_{\cA}\right\rbrace.
\end{equation}
If the optimum is not unique, we consider $\bx^\star$ to be one of largest atomic norm. 
We now present the convergence results for the Matching Pursuit algorithm variants.
\begin{theorem}
\label{thm:inexactMP}
Let $\cA \subset \cH$ be a bounded and symmetric set,  and let $f \colon \cH \to \R$ be $L$-smooth w.r.t. a given norm~$\|.\|$, over~$\rho\conv(\cA)$ with $\rho<\infty$.
Then, Norm-Corrective Matching Pursuit 
(Algorithm~\ref{algo:generalgreedy}), converges for $t \geq 0$ as 
\begin{equation*}
f(\bx_t) - f(\bx^\star) \leq \frac{4\left(\frac2\delta L \rho^2 \radius_{\|.\|\!}(\cA)^2 +\varepsilon_0\right)}{\delta t+4}
\end{equation*}
where $\varepsilon_0 := f(\bx_0) - f(\bx^\star)$ is the initial error in objective, and $\delta \in (0,1]$ is the relative accuracy of the employed approximate \lmo~\eqref{eq:inexactLMOMP}.
\end{theorem}
The proof of Theorem \ref{thm:inexactMP} extends the FW convergence analysis from $\conv(\cA)$ to $\lin(\cA)$ by rescaling $\conv(\cA)$ so that it includes $\bx^\star$ and $\bx_t$ for all $t \leq T$, the reason for which the rate in Theorem~\ref{thm:inexactMP} depends on the upper bound $\rho$ on the atomic norm of $\bx^\star$ and $\bx_t$, $t \leq T$. 
The relationship between Norm-Corrective FW and Norm-Corrective GMP is systematically studied in Section~\ref{sec:FW=MP}.

Using well-known results from convex optimization, we can particularize Theorem ~\ref{thm:inexactMP} for $f(\bx) := \frac{1}{2}\| \by - \bx\|^2$ and obtain iterate-independent constants (i.e., constants independent of~$\rho$) as follows.

\begin{definition}
The effective inradius of a convex set $\cA$, denoted by $\inr(\cA)$, is the radius of the largest $d$-dimensional Euclidean ball which can be inscribed in $\cA$, where $d$ is the dimension of the subspace spanned by $\lin(\cA)$.
\end{definition}

\begin{corollary} \label{cor:sublinleastsq}
Let $\cA \in \bbR^d$ be a finite symmetric set of atoms, or the convex hull of a finite set of atoms, and let $\tilde \rho \ge \max\left\lbrace \|\bx^\star\|, \|\bx_{0}\|,\ldots,\|\bx_T\|\right\rbrace$, $\tilde \rho < \infty$. Then, under the conditions of Theorem \ref{thm:inexactMP},  Algorithm~\ref{algo:generalgreedy} converges both with $f(\bx_t) - f(\bx^\star)\leq \frac{2 \tilde \rho^2 \diam_{\|.\|\!}(\cA)^2}{\delta^2\inr(\conv (\cA))^2 (t + 2)}$. If further $f(\bx) := \frac{1}{2}\| \by - \bx\|^2$, then $\tilde \rho$ can be replaced by $\|\by \|$.
\end{corollary}

The effective inradius $\inr(\conv(\cA))$ generally depends on the ambient space dimension $d$. For example, the effective inradius of the L1-ball scales as $O(\sqrt{d})$. Hence, if $\cA$ is the L1-ball, Corollary \ref{cor:sublinleastsq} tells us that we need to take $T$ at least on the order of $d$ to obtain an $O(1)$ error $f(\bx_t) - f(\bx^\star)$.
\vspace{-2mm}

\section{Linear Convergence Rates}
\vspace{-2mm}
It is possible to obtain faster convergence rates for some classes of objective functions, still over arbitrary dictionaries. In this section, we present linear convergence rates for our generalized matching pursuit, Algorithm
~\ref{algo:generalgreedy}. 
While linear rates have recently been demonstrated for Frank-Wolfe algorithm variants for strongly convex objectives by~\citep{LacosteJulien:2015wj}, we are not aware of any existing \emph{explicit} linear convergence rates for matching pursuit algorithms (see Section \ref{sec:priorwork} for a discussion).

We begin our analysis by proposing a new geometric complexity measure of the atom set which we call the \textit{minimal intrinsic directional width}. It builds upon the classic geometric width as follows:

\begin{definition}\label{def:feasible_dir}
The \emph{directional width of a set $\cA$} as a function of a given non-zero vector $\bd$ is defined as
\[
W_\cA(\bd) := \max_{\bz\in\cA} \, \textstyle\langle \frac{\bd}{\| \bd\|} ,\bz \rangle.\vspace{-2mm}
\]
\end{definition}
In general, the directional width can be zero depending on the choice of $\bd$.
Building upon the the concept of directional width, we are ready to define our main complexity constant, which will be crucial to our linear convergence guarantees.
\begin{definition}\label{def:mDW}
Given a bounded set $\cA$, we define its \emph{minimal intrinsic directional width} as 
\[
\mdw := \min_{ \substack {\bd\in\lin(\cA)\\\bd \neq \mathbf{0} }}W_\cA(\bd)  \ .\vspace{-2mm}
\]
\end{definition}

A crucial aspect of the preceding definition is that only directions in $\lin(\cA)$ are allowed, hence the name intrinsic. If the minimum was not over $\vd \in \lin(\cA)$, the width would be zero whenever $\cA$ does not span the ambient space. 

\paragraph{Properties.}
Note that $\mdw>0$ implies that the origin is in the relative interior of $\conv(\cA)$ and hence the atomic is well defined $\forall \bx \in \lin(\cA)$ (which ensures that $\rho<\infty$). Furthermore, note how for a fixed sequence of iterates and $\bx^\star$ the value of $\rho$ is a monotone decreasing function of the $\mdw$. 
Moreover, any symmetric set satisfies the property $\mdw>0$. 
For example, the L1 ball in $\bbR^d$ has $\mdw=\frac{1}{\sqrt{d}}$.
The quantity $\mdw$ is meaningful for both undercomplete and overcomplete, possibly continuous, atom sets,  and plays a similar role as the coherence in coherence-based convergence analysis of MPs (this is discussed in more detail at the end of this section).

We now present our main linear convergence result for optimization over the linear span of atoms as defined in \eqref{eq:linprob}.
As we will only consider strongly convex objective functions~$f$, the optimum $\bx^\star$ is unique here, as opposed to the general context of our sub-linear rates.

\begin{theorem}\label{thm:linearRateMPinexact}
Let $\cA \subset \cH$ be a bounded set such that $\mdw>0$, 
and let the objective function $f \colon \cH \to \R$ be both globally $L$-smooth and globally $\mu$-strongly convex w.r.t. a given norm~$\|.\|$ over $\rho \conv(\cA)$. 
 Then, for $t\geq 0$, the suboptimality of the iterates of Algorithm~\ref{algo:generalgreedy} decays exponentially as
\vspace{-2mm}
\begin{equation*}
\varepsilon_{t+1}
\leq \left(1- \delta^2\frac{\mu\mdw^2}{L\radius_{\|.\|\!}(\cA)^2} \right)\varepsilon_{t},
\end{equation*}
where $\varepsilon_t := f(\bx_t) - f(\bx^\star)$ is the suboptimality at step $t$, and $\delta \in (0,1]$ is the relative accuracy parameter of the employed approximate \lmo~\eqref{eq:inexactLMOMP}.
\vspace{-1mm}
\end{theorem}

Even though $\mdw$ can take on values larger than $1$ (depending on $\cA$) the rate in Theorem \ref{thm:linearRateMPinexact} is always valid as $\mdw/\radius_{\|.\|\!}(\cA) < 1$ for any non-empty $\cA$.

We present an additional illustrative experiment measuring the practical dependence of the convergence upon the defined $\mdw$ quantity in Appendix~\ref{sect:experiment}.
\vspace{-1mm}

\paragraph{Lower Bounds.}
We continue by presenting a lower bound on the decay of the suboptimality of the iterates for GMP. This lower bound depends on the width $W_\cA$, which shows that this quantity plays a fundamental role for the convergence of GMP. 
We first consider the general strongly convex and smooth functions and the particularize the result for the least-squares function $f(\bx) := \frac12 \| \by - \bx\|^2$, which allows to compute the update in closed-form. Furthermore, we consider only the case of the exact oracle ($\delta=1$ in Equation~\eqref{eq:inexactLMOMP}). 
\begin{theorem}\label{thm:lowerBoundLinearRateMPinexact}
Assume that $\bx^\star \in \lin(\cA)$ and let $\bz_t$ be the atom selected at iteration $t$ by the \lmo. 
Then, under the assumptions of Theorem \ref{thm:linearRateMPinexact}, the suboptimality of the iterates of Algorithm~\ref{algo:generalgreedy} Variant 0 with an exact \lmo ($\delta=1$) does not decay faster than
\vspace{-1mm}
\[
\varepsilon_{t+1}
\geq \left( 1 - \frac{W_\cA (-\nabla f(\bx_t))^2}{\|\bz_t\|^2} 
\frac{2L-\mu}{\mu}\right) \varepsilon_t
\vspace{-1mm}
\]
\end{theorem}
Note that the lower bound on the exponential decay given in Theorem~\ref{thm:lowerBoundLinearRateMPinexact} depends on the iteration $t$. 
We now particularize the result for the least-squares function.

\begin{corollary}\label{Cor:lowerBoundEx}

Let $\cA := \{\pm \be_i \} \subset\R^d$ be the vertices of the L1 ball. Suppose we are minimizing $f(\bx) := \frac12 \|\by-\bx\|^2$ over the linear span of $\cA$ with $\by\in\R^d$. Let $\bx_0$ be the starting point of the Matching Pursuit Algorithm and assume that $\forall i \in [d] \ (\bx_{0})_i\neq y_i$. Then\vspace{-2mm}
\[
\varepsilon_{t+1} 
\geq \left(1-\frac{1}{d-t}\right)\varepsilon_t \ .
\vspace{-3mm}
\]
\end{corollary}
This result is discussed in more detail in Appendix~\ref{lowerBoundEx}.
\vspace{-2mm}

\paragraph{Relationship between $\mdw$ and cumulative coherence.} It is interesting to compare the rate in Theorem~\ref{thm:inexactMP} with the coherence-based rates from the literature, such as \cite{Gribonval:2006ch}. 
In order to relate the two notions of cumulative coherence and directional width, we need some additional assumptions. 
We only consider the least-squares function in $\bbR^d$ and assume that its minimizer over $\bbR^d$ lies in the span of the atom set $\cA$. Further, we require symmetry so that the definition of \lmo given in Equation~\eqref{eq:FWLMO} is equivalent (up to the sign) to the one used for MP in \cite{Gribonval:2006ch}.

\begin{theorem}\label{thm:coherenceVSdw}
Let $\cA \subset \R^d$ be a symmetric set of $2n$ atoms with $\|\bs\|_2 = 1$ for all $\bs \in \cA$. Let $\cB$ be a set such that $\cA=\cB\cup-\cB$ with $\cB\cap-\cB=\emptyset$ and $|\cB|=n$. Then,  the cumulative coherence of the set $\cB$, defined as $\mu(\cB,m):=\max_{\cI \subset \cB, | \cI | = m} \max_{\bs_i \in \cB \backslash \cI} \sum_{\bs_j \in \cI} | \langle \bs_i, \bs_j \rangle |$, $m < n$, is lower-bounded as $\mu(\cB,n-1) \geq 1- n \cdot \mdw^2$.
\end{theorem}
\vspace{-1mm}

In essence, Theorem~\ref{thm:coherenceVSdw} shows that if the directional width is close to zero, the cumulative coherence is close to 1 with a factor that depends on $n$. 
Note that by increasing the number of atoms, both the cumulative coherence and $\mdw$ grow. Recall that when the cumulative coherence is 1, according to the rate for MP in \cite{Gribonval:2006ch} there is no linear convergence. 
Furthermore, our rate is more robust than the one in \cite{Gribonval:2006ch} in the following sense. An adversary could add an atom to the dictionary, making the coherence 1. In contrast, adding an atom cannot make $\mdw = 0$. 
In addition, if the atom is added so that $\mdw$ is arbitrarily small, the cumulative coherence is arbitrarily close to 1 by Theorem~\ref{thm:coherenceVSdw}. Finally, the linear rate for MP presented in \cite{Gribonval:2006ch} assumes that the optimum can be represented exactly using $m$ atoms. Therefore, the rate depends on $\mu(\cB,m-1)$ while $\mdw$ can be compared only to the cumulative coherence of the whole set (i.e., $\mu(\cB,n-1)$) since it is an intrinsic property of the atom set.
\vspace{-2mm}
\section{Affine Invariant Algorithms and Rates}
\label{sec:affineInvariantAlgorithms}
\vspace{-2mm}
We now present affine invariant versions of Algorithms~\ref{algo:normCorrectiveFW} and \ref{algo:generalgreedy}, along with sub-linear and linear convergence guarantees. An optimization method is called \emph{affine invariant} if it is invariant
under affine transformations of the input problem: If one chooses any
re-parameterization of the domain~$\domain$ by a \emph{surjective} linear or
affine map $\bM:\hat\domain\rightarrow\domain$, then the ``old'' and ``new''
optimization problems $\min_{\bx\in\domain}f(\bx)$ and
$\min_{\hat\bx\in\hat\domain}\hat f(\hat\bx)$ for $\hat f(\hat\bx):=f(\bM\hat\bx)$
look the same to the algorithm.

\vspace{-2mm}
\subsection{Affine Invariant Frank-Wolfe}
\vspace{-2mm}
To define an affine invariant upper bound on the objective function $f$, we use the affine invariant  definition of the \emph{curvature constant} from \cite{Jaggi:2013wg}
\vspace{-2mm}
\begin{equation}
\label{def:Cf}
\Cf := \sup_{\substack{\bs\in\cA,\, \bx \in \conv(\cA)
 \\ \gamma \in [0,1]\\ \by = \bx + \gamma(\bs- \bx)}} \frac{2}{\gamma^2} D(\by,\bx),
\vspace{-1mm}
\end{equation}
where for cleaner exposition, we have used the shorthand notation $D(\by,\bx)$ to denote the difference of $f(\by)$ and its linear approximation at $\bx$, i.e.,
\begin{equation*}
D(\by,\bx) :=f(\by) - f(\bx)- \langle \by -\bx, \nabla f(\bx)\rangle.
\end{equation*}
Bounded curvature $\Cf$ closely corresponds to smoothness of the objective~$f$. 
More precisely, if~$\nabla f$ is $L$-Lipschitz continuous on $\conv(\cA)$ with
respect to some arbitrary chosen norm $\norm{.}$, i.e., $\norm{\nabla f(\bx) - \nabla f(\by)}_* \leq L \norm{\bx-\by}$, where $\|.\|_*$ is the dual norm of $\|.\|$, then
\begin{equation} \label{eq:CfBound}
\Cf \le L \diam_{\norm{.}}(\cA)^2  \ ,
\end{equation}
where $\diam_{\norm{.}}(.)$ denotes the $\norm{.}$-diameter, see \cite[Lemma
7]{Jaggi:2013wg}. The curvature constant $\Cf$ is affine invariant, does not depend on any
norm. It combines the complexity of the domain $\conv(\cA)$ and the curvature of the objective function~$f$ into a single quantity. 

We are now ready to present the affine invariant version of the Norm-Corrective Frank-Wolfe algorithm (Algorithm \ref{algo:normCorrectiveFW}).
\vspace{-2mm}

\begin{algorithm}[h]
  \caption{Affine Invariant Frank-Wolfe} 
  \label{algo:affineInvariantCorrectiveFW}
\begin{algorithmic}
  \STATE  Same as Algorithm~\ref{algo:normCorrectiveFW}, except,
  \STATE 5.\qquad $\gamma := \clip_{[0,1]}\!\big[  \langle -\nabla f(\bx_t), \bz_t - \bx_t\rangle / \Cf \big]$
    \STATE 6.\qquad Update $\bx_{t+1}:= \bx_t + \gamma (\bz_t - \bx_t)$
\end{algorithmic}
\end{algorithm}

The following theorem characterizes the sub-linear convergence rate of Algorithm \ref{algo:affineInvariantCorrectiveFW}.

\begin{theorem}
\label{thm:sublinearFWAffineInvariant}
Let $\cA \subset \cH$ be a bounded set and let $f \colon \cH \to \R$ be a convex function with curvature $\Cf$ over~$\cA$ as defined in~\eqref{def:Cf}.
Then, the Affine Invariant Frank-Wolfe algorithm (Algorithm~\ref{algo:affineInvariantCorrectiveFW}) 
converges for $t \geq 0$ as\vspace{-1mm}
\begin{equation*}
f(\bx_t) - f(\bx^\star) \leq \frac{2\left(\frac1\delta \Cf+\varepsilon_0\right)}{\delta t+2}
\end{equation*}
where $\varepsilon_0 := f(\bx_0) - f(\bx^\star)$ is the initial error in objective, and $\delta \in (0,1]$ is the accuracy parameter of the employed approximate \lmo  (Equation~\eqref{eq:inexactLMOFW}).
\vspace{-1mm}
\end{theorem}

\subsection{Affine Invariant Generalized Matching Pursuit}
\vspace{-2mm}
To design an affine invariant MP algorithm 
we will rely on the following slight variation of $\Cf$ (defined in \eqref{def:Cf}) using $\by = \bx + \gamma \bs$ instead of $\by = \bx + \gamma(\bs- \bx)$, i.e.,\vspace{-1mm}
\begin{equation}\label{eq:CfMP}
\CfMP := \sup_{\substack{\bs\in\cA,\, \bx \in \conv(\cA) \\ \gamma \in [0,1]\\ \by = \bx + \gamma \bs}} \frac{2}{\gamma^2} D(\by,\bx).
\end{equation}
Throughout this section, we again assume availability of a finite constant $\rho>0$ as an upper bound of the atomic norms~$\|.\|_\cA$ of the optimum $\bx^\star$, as well as the iterate sequence $(\bx_t)_{t=0}^T$ up to the current iteration, as defined in~\eqref{eq:rho}.
We now present the affine invariant version of the Norm-Corrective GMP algorithm (Algorithm \ref{algo:generalgreedy}, Variant 0) in Algorithm \ref{algo:affineInvariantMP}. 
The algorithm uses the bounded curvature $\CfMPr$ over the rescaled set $\rho \conv(\cA)$, rather than $\conv(\cA)$.
\vspace{-2mm}

\begin{algorithm}[h]
\caption{Affine Invariant Generalized Matching Pursuit}
  \label{algo:affineInvariantMP}
\begin{algorithmic}
  \STATE Same as Algorithm~\ref{algo:generalgreedy} except,
  \STATE 5. \qquad Variant 1: $\gamma := \langle -\nabla f(\bx_t), \rho^2 \bz_t
   \rangle / \CfMPr$
  \STATE 6. \qquad \qquad \qquad Update $\bx_{t+1}:= \bx_t + \gamma \bz_t$
  \STATE 5. \qquad Variant 2: $\bx_{t+1} = \argmin_{\bx\in\lin(\cS)} f(\bx)$

\end{algorithmic}
\end{algorithm}

A sub-linear convergence guarantee for Algorithm \ref{algo:affineInvariantMP} is presented in the following theorem.

\begin{theorem}
\label{thm:sublinearMPAffineInvariant}
Let $\cA \subset \cH$ be a bounded and symmetric set such that $\rho < \infty$. 
Then, Algorithm~\ref{algo:affineInvariantMP} converges for $t \geq 0$ as \vspace{-2mm}
\[
f(\bx_t) - f(\bx^\star)\leq \frac{2\left(\frac2\delta \CfMPr+\varepsilon_0\right)}{\frac\delta2 t+2} ,
\vspace{-1mm}
\]
where $\delta \in (0,1]$ is the relative accuracy parameter of the employed approximate \lmo~\eqref{eq:inexactLMOMP}.
\end{theorem}

Exact knowledge of $\CfMPr$ is not required: The same theorem also holds if any upper bound on $\CfMPr$ is used in the algorithm and resulting rate instead. Note further that the convergence guarantee in Theorem \ref{thm:sublinearMPAffineInvariant} is linear invariant only as the assumption of $\cA$ being symmetric precludes affine maps involving translations.

We proceed by establishing a linear convergence guarantee for Algorithm \ref{algo:affineInvariantMP}. For lower-bounding the error at iteration~$t$, we need to define an affine invariant analog of strong convexity over the requisite domain. The following positive step size quantity relates the dual certificate value of the descent direction $\bx^\star - \bx$ with the MP selected atom, 
\begin{equation}
\label{def:stepsizeGamma}
\gamma(\bx,\bx^\star) := \frac{ \langle-\nabla f(\bx), \bx^\star - \bx \rangle}{\langle-\nabla f(\bx),  \bs(\bx) \rangle} \ ,
\end{equation}
for $\bs(\bx) := \argmin_{\bs\in\cA}\ \langle \nabla f(\bx), \bs\rangle$.

A quantity similar to~\eqref{def:stepsizeGamma} but using a different direction $\bs(\bx)$ was also used by~\citep{LacosteJulien:2015wj} to study linear convergence of FW variants. We now define the complexity measure $\muf$, which serves as an affine invariant notion of strong convexity of the objective $f$, over the domain $\conv(\cA)$.  
\vspace{-1mm}
\begin{equation}
\label{def:muf}
\muf := \inf_{\bx \in \conv(\cA)} \inf_{\substack{\bx^\star \in \conv(\cA)\\ \langle\nabla f (\bx), \bx^\star - \bx \rangle < 0 }} \frac{2}{\gamma(\bx,\bx^\star)^2} D(\bx^\star,\bx).\vspace{-1mm}
\end{equation}
In the following, our results will depend on $\mufr$, which is this quantity $\muf$ taken over the scaled set $\rho \cA$ instead of $\cA$. This is analogous to the smoothness parameter $\CfMP$ as we have seen in the previous results. Theorem~\ref{thm:LinearMPAffineInvariant} characterizes the linear convergence of Algorithm \ref{algo:affineInvariantMP}.

\begin{theorem}
\label{thm:LinearMPAffineInvariant}
Let $\cA \subset \cH$ be a bounded 
set. 
Then, Algorithm~\ref{algo:affineInvariantMP} converges linearly as 
\vspace{-1mm}
\[
\varepsilon_{t+1} \leq \bigg(1 - \delta^2\frac{\mufr}{\CfMPr} \bigg) \varepsilon_t
\vspace{-1mm}
\]
where $\varepsilon_t := f(\bx_t) - f(\bx^\star)$ is the suboptimality at step $t$, and $\delta \in (0,1]$ is the relative accuracy parameter of the employed approximate \lmo~\eqref{eq:inexactLMOMP}.
\vspace{-1mm}
\end{theorem}

\paragraph{Discussion:}

Note that the new affine invariant convergence rates in Theorems~\ref{thm:sublinearFWAffineInvariant}, \ref{thm:sublinearMPAffineInvariant}, and \ref{thm:LinearMPAffineInvariant} do imply the rates presented earlier for their norm-based algorithm counterparts in Theorems~\ref{thm:inexactFW},~\ref{thm:inexactMP}, and~\ref{thm:linearRateMPinexact}, respectively, 
for any choice of norm.
This follows simply establishing the relationships between $\Cf$ and $L$ (see \eqref{eq:CfBound}) and accordingly for the strong convexity notion $\muf$ compared to $\mu$. For the latter, it is not hard to show that if $\mdw>0$, $\muf \geq \mu  \mdw^2$, see Lemma \ref{lemma:MUFwithMDW} in the appendix.
The affine invariant convergence guarantees are therefore more general than the norm-based ones.
\vspace{-1mm}

\section{On the Relationship Between Matching Pursuit and Frank-Wolfe}\label{sec:FW=MP}
\vspace{-2mm}
The sub-linear convergence rates for MP and FW are related by the constant $\rho$ that  essentially simulates a ``blown up'' set in which the analysis of FW can be applied. In this section, we explore this relationship.

Let $\alpha\cA := \{ \alpha\bz \ |\ \bz\in\cA\}$, and assume $\alpha\geq\frac{\rho}{\delta}$. We will consider Norm-Corrective FW (Algorithm~\ref{algo:normCorrectiveFW}) on the set $\alpha\cA$ and analyze its behavior when $\alpha$ grows to infinity, relating the iterates of Algorithm~\ref{algo:normCorrectiveFW} with the ones of Algorithm~\ref{algo:generalgreedy}. 
\begin{theorem}\label{thm:stepFWMPshort}
Let $\cA \subset \cH$ be a bounded set and let $f \colon \cH \to \R$ be a $L$-smooth convex function.
Let  $\alpha>0$ and let us fix $t>0$ with iterate $\bx_{t}$. There exists a polynomial function $\theta(f,\bx_t,\alpha)$ such that if $\theta(f,\bx_t,\alpha)\leq 1$ the new iterate $\bx_{t+1}^\FWa$ of Frank-Wolfe (Algorithm~\ref{algo:FW}) using the set $\alpha\cA$ converges to the new iterate $\bx_{t+1}^\MP$ of Matching Pursuit (Algorithm~\ref{algo:generalgreedy}) applied on the linear span of the set $\cA$ with rate:
\vspace{-2mm}
\[
 \big\|\bx_{t+1}^\FWa - \bx_{t+1}^\MP\big\| \in O\left(\frac{1}{\alpha}\right). 
\]
In particular, when $\alpha$ grows to infinity, the condition $\theta(f,\bx_t,\alpha)\leq 1$ always holds (for all steps $t$). 
Otherwise, the difference of the iterates satisfies
$
 \big\|\bx_{t+1}^\FWa - \bx_{t+1}^\MP\big\| \in O\left({\alpha}\right) .
$
\end{theorem}

Our analysis shows that, in some sense, FW can be suitable to solve the optimization problem \eqref{eq:linprob}.
Indeed, if we knew the atomic norm of the iterates and the optimum in advance (which is usually not the case in practice), we could just consider a large enough convex set and run FW (Algorithm~\ref{algo:FW}) on $\alpha\cA$ with $\alpha =\rho$ ($\rho$ as defined in Section \ref{sec:sublin}) for an exact oracle (this can be seen in the proof of Theorem \ref{thm:inexactMP}).

\vspace{-2mm}

\section{Relation to Prior Generalizations of MP} \label{sec:priorwork}
\vspace{-2mm}

\citet{ShalevShwartz:2010wq}, \citet{Temlyakov:2013wf, Temlyakov:2014eb, Temlyakov:2012vg}, and \citet{nguyen2014greedy} propose and study algorithms similar to Algorithm \ref{algo:generalgreedy}---although using the objective function directly in the update step instead of a quadratic upper bound---for the optimization of smooth functions on Banach spaces. \citet{nguyen2014greedy} consider orthonormal bases as dictionaries only. The sub-linear rates derived in \cite{Temlyakov:2013wf, Temlyakov:2014eb, Temlyakov:2012vg,nguyen2014greedy} are similar to ours, whereas the linear rates in \cite{Temlyakov:2013wf, Temlyakov:2014eb} critically rely on incoherence and approximate sparsity (of the optimal solution) assumptions. Most importantly, these linear rates only hold for a finite number of iterations that is related to the sparsity level of the solution. Note that the linear rates for (least-squares) MP and OMP in \cite{Gribonval:2006ch} hold under similar incoherence and sparsity assumptions. The linear rates for a fully-corrective GMP variant in \cite{ShalevShwartz:2010wq} holds under a (sparsity-based) restricted strong convexity assumption. 

Much more general rates are known for the class of random pursuit algorithms --- which are derivative-free and use random directions instead of an LMO --- as shown by \cite{Stich:2013ji}. These rates only apply to the unconstrained setting $\cA=\R^d$ (so do not cover the general Hilbert-space case) and do scale with the dimension as $\Theta(d)$, whereas our rates are dimension independent (but need an LMO).

In the statistics community, very related methods are studied under the names of, e.g., forward selection and stage-wise algorithms, see \citep{tibshirani2015general} for a recent overview.  The stage-wise framework considers the evolution of the solution---the regularization path---as the scaling of the constraint set grows (or the corresponding regularizer weakens). Our results can help to also equip such algorithms with explicit convergence rate, at any fixed regularizer value.

To the best of our knowledge, the only prior work on greedy optimization that also relies on a quadratic upper bound of the (smooth) objective function in the update step is \cite{yaogreedy}. However, \cite{yaogreedy} specifically targets matrix completion, considers the set of unit norm rank-one matrices as dictionary only, and obtains problem-specific (i.e., matrix-specific) and iterate-dependent (implicit) sub-linear and linear rates. Hence, the setting considered here, i.e., functions on Hilbert spaces and general dictionaries, and the linear rate depending only on geometric properties of the dictionary enjoy much higher generality.

Finally, recovery guarantees for sparse solutions of convex optimization problems using generalized MPs were proposed, e.g., in \cite{blumensath2008gradient, Zhang:2011uy}.

\paragraph{Acknowledgments:} The authors thank Zaid Harchaoui and Gunnar R{\"a}tsch for fruitful discussions. FL is supported by the Max Plank-ETH Center for Learning Systems.

\newpage
{\small
\bibliographystyle{icml2016}
\bibliography{bibliography}
}

\appendix
\clearpage
\section{An Illustrative Experiment}\label{sect:experiment}
\vspace{-2mm}
In this section we numerically investigate the tightness of the linear rate presented in Theorem~\ref{thm:linearRateMPinexact} and illustrate the impact of the $\mdw$ on the empirical rate of Algorithm~\ref{algo:generalgreedy} (Variant 0, with exact \lmo). The experiment setup is the following. We minimize the function $f(\bx) = \|\bx^\star-\bx\|^2$ over the set\vspace{-2mm}
\[
\cA := \{\cA_\theta \cup -\cA_\theta \}
\ \text{ where } \ 
\cA_\theta := \Big\{ {1\choose 0}, {\cos\theta \choose \sin\theta} \Big\}\vspace{-2mm}
\]
with $\theta\in(0,\pi/2]$ and $\bx^\star:=\left(-1,1\right)^\top $.

This choice for the set $\cA$ allows to control $\mdw$ by acting on $\theta$. In Figure~\ref{pic:ratio} we plot the ratio between the theoretical linear rate in Theorem~\ref{thm:linearRateMPinexact} and the empirical rate, averaged over 20 random initializations chosen within $\conv(\cA)$.
The rate is tight when the bound on the error decrease matches the empirical decay, i.e., when their ratio is equal to 1. It can be seen that the upper bound in Theorem~\ref{thm:linearRateMPinexact} is within a factor 2.5 of the empirical rate on average.
\begin{figure}[h]
\center\includegraphics[width=5cm]{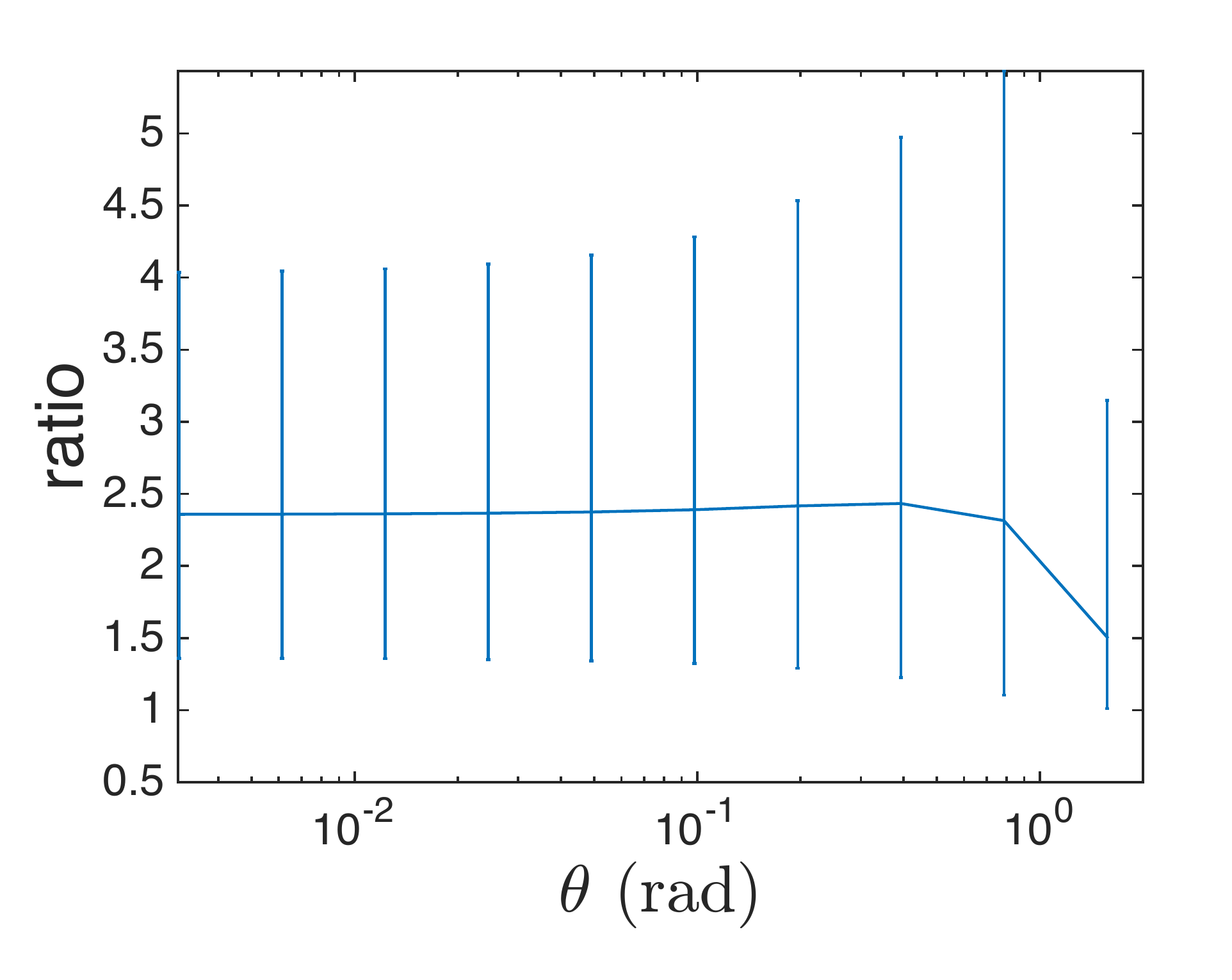}\vspace{-1em}
\caption{Ratio of theoretical rate and empirical rate, minimum, maximum and average ratio over 20 runs from random starting point in $\conv(\cA)$.}\label{pic:ratio}
\end{figure}

\section{Proofs of the Main Results}
An optimization method is called \emph{affine invariant} if it is invariant
under affine transformations of the input problem: If one chooses any
re-parameterization of the domain~$\domain$ by a \emph{surjective} linear or
affine map $\bM:\hat\domain\rightarrow\domain$, then the ``old'' and ``new''
optimization problems $\min_{\bx\in\domain}f(\bx)$ and
$\min_{\hat\bx\in\hat\domain}\hat f(\hat\bx)$ for $\hat f(\hat\bx):=f(\bM\hat\bx)$
look the same to the algorithm~\cite{Jaggi:2013wg}.

\subsection{Sublinear FW rate} \label{sec:sublinfwaffinv}
\begin{reptheorem}{thm:sublinearFWAffineInvariant}
Let $\cA \subset \cH$ be a bounded set and let $f \colon \cH \to \R$ be a convex function with curvature $\Cf$ over~$\cA$ as defined in~\eqref{def:Cf}.
Then, the Frank-Wolfe method (Algorithm~\ref{algo:FW}) with step-size variants 1 and 2, converges for $t \geq 0$ as
\begin{equation*}
f(\bx_t) - f(\bx^\star) \leq \frac{2\left(\frac1\delta \Cf+\varepsilon_0\right)}{\delta t+2}
\end{equation*}
where $\varepsilon_0 := f(\bx_0) - f(\bx^\star)$ is the initial error in objective, and $\delta \in (0,1]$ is the accuracy parameter of the employed approximate \lmo  (Equation~\eqref{eq:inexactLMOFW}).
\end{reptheorem}
\begin{proof}
At iteration $t$, let $\bz_t\in\cA$ be the atom selected by the Approx-\lmo. 
The key to the proof is to use the definition of the curvature constant $\Cf$ as to give an affine invariant upper bound on the objective $f$:
\begin{align}
\label{eq:affineInvariantUpperbound}
f(\bx_t + \gamma (\bz_t - \bx_t)) \leq\ &\notag\\
g_{\bx_{t}}(\bx_t + \gamma (\bz_t - \bx_t)) :=&\\
 f(\bx_t) + \gamma \langle \nabla f(\bx_t), \bz_t & - \bx_t\rangle +\frac{\gamma^2} {2}  \Cf. \notag
\end{align}
By computing the closed-form-solution for $\gamma$ minimizing the right hand side, we have
\begin{equation}
\gamma  = \frac{\langle -\nabla f(\bx_t), \bz_t - \bx_t\rangle}{\Cf} \ .
\end{equation}
This is exactly the update-step used by variant 2 of the FW algorithm (Algorithm \ref{algo:FW}).
In other words, the algorithm in each iteration performs a step as to minimize this upper bound to~$f$, over the line segment $[\bx_t, \bz]$.

Writing $\varepsilon_t := f(\bx_t) - f(\bx^\star)$ for the suboptimality, we apply the certificate property of the duality gap, 
$\varepsilon_t 
\leq \langle -\nabla f(\bx_{t}),\bz_{t} - \bx_{t}\rangle$. Combining this with the given approximation quality $\delta \in (0,1]$ of the used Approx-\lmo, we have 
\[
\delta \varepsilon_t \le \langle -\nabla f(\bx_{t}),\bz_{t} - \bx_{t}\rangle \ .
\]
 
Continuing from~\eqref{eq:affineInvariantUpperbound}, 
\begin{eqnarray*}
 \varepsilon_{t+1} &\leq \varepsilon_{t} + \min_{\gamma\in[0,1]}\left\lbrace -\gamma \delta \varepsilon_{t} + \frac{\gamma^2}{2}\Cf\right\rbrace\\
 & \leq \varepsilon_{t} - \frac{2}{\delta t+2}\delta \varepsilon_{t} + \frac{1}{2}\left(\frac{2}{\delta t+2}\right)^2 \Cf
\end{eqnarray*}

Finally, we show by induction
 \begin{equation*}
 \varepsilon_t \leq 2\frac{\left(\frac{1}{\delta} \Cf + \varepsilon_0\right)}{\delta t+2}.
 \end{equation*}
for $t \geq 0$.

When $t=0$ we get $\varepsilon_0\leq \left(\frac1\delta \Cf+\varepsilon_0\right)$. Therefore, the base case holds. We now prove the induction step assuming $\varepsilon_t \leq \tfrac{2\left(\tfrac1\delta\Cf+\varepsilon_0\right)}{\delta t+2}$.
\begin{align*}
\varepsilon_{t+1} &\leq \left(1-\tfrac{2\delta}{\delta t + 2}\right)\varepsilon_{t} + \tfrac12\Cf\left(\tfrac{2}{\delta t+2}\right)^2\\
&\leq \left(1-\tfrac{2\delta}{\delta t + 2}\right)\tfrac{2\left(\tfrac1\delta\Cf+\varepsilon_0\right)}{\delta t+2} \\
&\quad+ \tfrac{1}{2}\left(\tfrac{2}{\delta t+2}\right)^2\Cf + \tfrac{2}{(\delta t+2)^2}\delta\varepsilon_0\\
&= \tfrac{2\left(\tfrac1\delta\Cf+\varepsilon_0\right)}{\delta t+2}\left(1-\tfrac{2\delta}{\delta t +2}+\tfrac{\delta}{\delta t +2}\right)\\
&\leq \tfrac{2\left(\tfrac1\delta\Cf+\varepsilon_0\right)}{\delta(t+1)+2}
\end{align*}

The same rate will hold for variant 1 (line-search on the true~$f$) of Algorithm \ref{algo:FW}, since the resulting objective per step will always be at least as good as the pre-determined step-size as in variant 2.
\end{proof}

\paragraph{Norm-based Variants.}
Note that for any choice of norm $\|.\|$, the analogous sublinear convergence rates do hold if the objective function is L-smooth (i.e., $\nabla f$ is $L$-Lipschitz) with respect to the norm $\|.\|$.

\begin{reptheorem}{thm:inexactFW}
Let $\cA \subset \cH$ be a bounded set and let $f \colon \cH \to \R$ be $L$-smooth w.r.t. a given norm $\|.\|$, over~$\conv(\cA)$.
Then, the Frank-Wolfe method (Algorithm~\ref{algo:FW})
, as well as Norm-Corrective Frank-Wolfe (Algorithm~\ref{algo:normCorrectiveFW}), converge for $t \geq 0$ as\vspace{-2mm}
\begin{equation*}
f(\bx_t) - f(\bx^\star) \leq \frac{2\left(\frac1\delta L\diam_{\|.\|\!}(\cA)^2 +\varepsilon_0\right)}{\delta t+2}
\end{equation*}
where $\varepsilon_0 := f(\bx_0) - f(\bx^\star)$ is the initial error in objective, and $\delta \in (0,1]$ is the accuracy parameter of the employed approximate \lmo  (Equation~\eqref{eq:inexactLMOFW}).
\end{reptheorem}
\begin{proof}
The proof for variant 3 of Algorithm~\ref{algo:FW} follows directly from the fact that any norm gives an upper bound on $\Cf$, as shown in \eqref{eq:CfBound}.

For variant 4 of Algorithm~\ref{algo:FW}, as well as for the Norm-Corrective Frank-Wolfe (Algorithm~\ref{algo:normCorrectiveFW}), the analogous convergence rate follows by using the quadratic upper bound
\begin{equation} 
 g_{\bx_{t}}(\bx) = f(\bx_{t}) + \langle\nabla f(\bx_{t}), \bx-\bx_{t}\rangle+\frac{L}{2}\|\bx-\bx_{t}\|^2
\end{equation}
instead of \eqref{eq:affineInvariantUpperbound} in the above proof.
\end{proof}

Also compare the above proof to Theorem C.1 \cite{LacosteJulien:2013ue}, which can be extended to the same algorithm variants as of our interest here.

\subsection{Sublinear MP rates}
\label{sec:proofSublinear}
\begin{reptheorem}{thm:sublinearMPAffineInvariant}
Let $\cA \subset \cH$ be a bounded and symmetric set and let $\rho := \max\left\lbrace \|\bx^\star\|_{\cA}, \|\bx_{0}\|_{\cA},\ldots,\|\bx_T\|_{\cA}\right\rbrace $.
Then, Algorithm~\ref{algo:affineInvariantMP} converges for $t \geq 0$ as 
\[
f(\bx_t) - f(\bx^\star)\leq \frac{4\left(\frac2\delta \CfMPr+\varepsilon_0\right)}{\delta t+4} ,
\]
where $\delta \in (0,1]$ is the relative accuracy parameter of the employed approximate \lmo~\eqref{eq:inexactLMOMP}.
\end{reptheorem}

\begin{proof}

Recall that $\tilde{\bz}_t$ is the atom selected in iteration $t$ by the approximate \lmo defined in \eqref{eq:inexactLMOMP}. We start by upper-bounding $f$ on $\rho \conv(\cA)$ using the definition of $\CfMPr$ as follows
\begin{eqnarray}
f(\bx_{t+1}) &\leq &  \min_{\gamma\in[0,1]} f(\bx_t) + \gamma \langle \nabla f(\bx_t), 
\rho \tilde{\bz}_t \rangle   + \frac{\gamma^2}{2} \CfMPr \nonumber \\
& \leq &  f(\bx_t) +  \min_{\gamma\in[0,1]} \left\lbrace - \frac{\delta}{2} \gamma \varepsilon_t + \frac{\gamma^2}{2} \CfMPr \right\rbrace, \label{eq:mpubeps}
\end{eqnarray}
where the first inequality holds for Algorithm~\ref{algo:affineInvariantMP} because the step size in Algorithm~\ref{algo:affineInvariantMP} is chosen by minimizing the RHS over $\gamma \in \R$. To get the second inequality note that both $-\bx_t$ and $\bx^\star$ are in $\rho\cdot \conv(\cA)$ by symmetry. We therefore have the following sequence of inequalities
 \begin{align}
 \langle \nabla f(\bx_{t}), - 2\frac{\rho}{\delta} \tilde{\bz}_t\rangle &= \langle \nabla f(\bx_{t}), - \frac{\rho}{\delta} \tilde{\bz}_t\rangle+\langle\nabla f(\bx_{t}), - \frac{\rho}{\delta} \tilde{\bz}_t\rangle\nonumber\\
 &\geq \langle \nabla f(\bx_{t}), - \rho\bz_t\rangle+\langle\nabla f(\bx_{t}), - \rho\bz_t\rangle\label{step:LMO_def_MPaff}\\
 &\geq \langle\nabla f(\bx_{t}),\bx_{t} - \bx^\star\rangle\label{step:symmetryaff}\\
 &\geq f(\bx_{t}) -f(\bx^\star)=: \varepsilon_{t}, \label{step:conv_f_MPaff}
\end{align}
where \eqref{step:LMO_def_MPaff} follows from the definition of inexactness of the \lmo (Equation~\eqref{eq:inexactLMOMP}) and \eqref{step:symmetryaff} from the fact that $-\rho\bz_t$ has the largest 
inner product with the positive gradient with respect to all the elements in $\conv(\rho\cA)$. Note that by the symmetry of $\cA$ and by definition of $\rho$ both $-\bx_t$ and $\bx^\star$ are in $\conv(\rho\cA)$.
Equation~\eqref{step:conv_f_MPaff} (known as weak duality) again follows from the convexity of $f$. 

Now, subtracting $f(\bx^\star)$ from both sides of \eqref{eq:mpubeps}, we get
\begin{eqnarray*}
 \varepsilon_{t+1} &\leq \varepsilon_{t} + \min_{\gamma\in[0,1]}\left\lbrace - \frac{\delta}{2} \gamma \varepsilon_{t} + \frac{\gamma^2}{2}\CfMPr\right\rbrace\\
 & \leq \varepsilon_{t} - \frac{2}{\delta' t+2}\delta' \varepsilon_{t} + \frac{1}{2}\left(\frac{2}{\delta' t+2}\right)^2 \CfMPr,
\end{eqnarray*}
where we set $\delta' := \delta/2$ and used $\gamma = \frac{2}{\delta' t+2} \in [0,1]$ to obtain the second inequality. Finally, we show by induction
 \begin{equation*}
 \varepsilon_t \leq \frac{4\left(\frac{2}{\delta} \CfMPr + \varepsilon_0\right)}{t+4} = 2\frac{\left(\frac{1}{\delta'} \CfMPr + \varepsilon_0\right)}{\delta' t+2}
 \end{equation*}
for $t \geq 0$.

When $t=0$ we get $\varepsilon_0\leq \left(\frac{1}{\delta'}\CfMPr+\varepsilon_0\right)$. Therefore, the base case holds. We now prove the induction step assuming $\varepsilon_t \leq \tfrac{2\left(\frac{1}{\delta'}\CfMPr+\varepsilon_0\right)}{\delta' t+2}$ as :
\begin{align*}
\varepsilon_{t+1} &\leq \left(1-\tfrac{2\delta'}{\delta' t + 2}\right)\varepsilon_{t} + \tfrac12 \CfMPr \left(\tfrac{2}{\delta' t+2}\right)^2\\
&\leq \left(1-\tfrac{2\delta'}{\delta' t + 2}\right)\tfrac{2\left(\frac{1}{\delta'}\CfMPr+\varepsilon_0\right)}{\delta' t+2} \\
&\quad+ \tfrac{1}{2}\left(\tfrac{2}{\delta' t+2}\right)^2\CfMPr + \tfrac{2}{(\delta' t+2)^2}\delta'\varepsilon_0\\
&= \tfrac{2\left(\frac{1}{\delta'}\CfMPr+\varepsilon_0\right)}{\delta' t+2}\left(1-\tfrac{2\delta'}{\delta' t +2}+\tfrac{\delta'}{\delta' t +2}\right)\\
&\leq \tfrac{2\left(\frac{1}{\delta'}\CfMPr+\varepsilon_0\right)}{\delta'(t+1)+2}.
\end{align*}

\end{proof}

We next explore the relationship of $\CfMPr$ and the smoothness parameter. Recall that $f$ is \emph{$L$-smooth} w.r.t. a given norm $\|.\|$ over a set $\cQ$ if 
\begin{equation}
\| \nabla f(\bx) - \nabla f(\by) \|_* \leq L \| \bx-\by \| \text{ for all } \bx,\by\in\cQ \ ,
\end{equation}
 where $\|.\|_*$ is the dual norm of $\|.\|$.

\begin{lemma}
\label{lem:CfwithRadius}
Assume $f$ is $L$-smooth w.r.t. a given norm $\|.\|$, over the set $\conv(\cA)$ with $\mdw>0$.
Then, 
\begin{equation}
 \CfMP \leq L \, \radius_{\norm{.}}(\cA)^2
\end{equation}
\end{lemma}
\begin{proof}
By the definition of smoothness of $f$ w.r.t. $\|.\|$, 
\begin{eqnarray*}
D(\by,\bx) \leq \frac{L}{2} \| \by - \bx\|^2.
\end{eqnarray*}

Hence, from the definition of $\CfMP$, 
\begin{eqnarray*}
\CfMP &\leq& \sup_{\substack{\bs\in\cA,\, \bx \in \conv{(\cA)} \\ \gamma \in [0,1]\\ \by = \bx + \gamma \bs}}
 \frac{2}{\gamma^2} \frac{L}{2}  \| \by - \bx\|^2 \\ 
&=& L \sup_{\bs \in \conv{(\cA)}} \, \| \bs\|^2 \\
&=& L \, \radius_{\norm{.}}(\cA)^2 \ . 
\end{eqnarray*}
\end{proof}
As an immediate corollary of the above lemma, we have that $\CfMPr \leq L \, \rho^2 \radius_{\norm{.}}(\cA)^2$ for any scaled set $\rho\cA$ and any norm, if $f$ is $L$-smooth w.r.t. that norm.

Related to our above complexity quantities $\Cf$ given in~\eqref{def:Cf} and $\CfMP$ given in~\eqref{eq:CfMP}, \citet{LacosteJulien:2015wj} have defined a slight variation called $\CfAW$, in order to bound the convergence rates for the Away-Step and Pairwise FW methods. $\CfAW$ is defined as
\begin{equation}
\label{def:CfA}
\CfAW := \sup_{\substack{\bs,\bv\in\cA,\, \bx \in \conv(\cA) \\ \gamma \in [0,1]\\ \by = \bx + \gamma(\bs - \bv)}} \frac{2}{\gamma^2}D(\by,\bx).
\end{equation}
Our $\CfMP$ can be considered as a variant of $\CfAW$ with the away atom $\bv$ fixed to be $0$. Thus, $\CfMP \leq \CfAW$. 

\paragraph{Norm-based Variants.}
Note that for any choice of norm $\|.\|$, the analogous sublinear convergence rates do hold if the objective function is L-smooth (i.e., $\nabla f$ is $L$-Lipschitz) with respect to the norm $\|.\|$.

\begin{reptheorem}{thm:inexactMP}
Let $\cA \subset \cH$ be a bounded and symmetric set,  and let $f \colon \cH \to \R$ be $L$-smooth w.r.t. a given norm~$\|.\|$, over~$\rho\conv(\cA)$, with $\rho<\infty$.
Then, Norm-Corrective Matching Pursuit 
(Algorithm~\ref{algo:generalgreedy}), converges for $t \geq 0$ as 
\begin{equation*}
f(\bx_t) - f(\bx^\star) \leq \frac{2\left(\frac2\delta L \rho^2 \radius_{\|.\|\!}(\cA)^2 +\varepsilon_0\right)}{\frac\delta2 t+2}
\end{equation*}
where $\varepsilon_0 := f(\bx_0) - f(\bx^\star)$ is the initial error in objective, and $\delta \in (0,1]$ is the relative accuracy of the employed approximate \lmo~\eqref{eq:inexactLMOMP}.
\end{reptheorem}
\begin{proof}
The proof follows directly from the fact that
$\CfMPr \leq L \, \rho^2 \radius_{\norm{.}}(\cA)^2$ for any scaled set $\rho\cA$, under smoothness of $f$, as shown in Lemma \ref{lem:CfwithRadius}.
\end{proof}

\subsection{Linear MP rates}

\begin{reptheorem}{thm:LinearMPAffineInvariant}
Let $\cA \subset \cH$ be a bounded 
set. 

Then, Algorithm~\ref{algo:affineInvariantMP} converges linearly as 
\[
\varepsilon_{t+1} \leq \bigg(1 - \delta^2\frac{\mufr}{\CfMPr} \bigg) \varepsilon_t
\]
where $\varepsilon_t := f(\bx_t) - f(\bx^\star)$ is the suboptimality at step $t$, and $\delta \in (0,1]$ is the relative accuracy parameter of the employed approximate \lmo  (Equation~\eqref{eq:inexactLMOMP}).
\end{reptheorem}
\begin{proof}
Using the definition of $\CfMPr$ we upper-bound $f$ on $\rho \conv(\cA)$ as follows
\begin{eqnarray*}
f(\bx_{t+1}) &\leq & \min_{\gamma \in [0,1]} f(\bx_t) + \gamma \langle \nabla f(\bx_t), \rho \tilde{\bz}_t \rangle   + \frac{\gamma^2}{2} \CfMPr \\
&= & \min_{\gamma \in \bbR} f(\bx_t) + \gamma \langle \nabla f(\bx_t), \rho \tilde{\bz}_t \rangle   + \frac{\gamma^2}{2} \CfMPr \\
& =& f(\bx_t) - \frac{\rho^2}{2\CfMPr} \left\langle \nabla f(\bx_t), \tilde{\bz}_t \right\rangle ^2.
\end{eqnarray*}
This upper bound holds for Algorithm~\ref{algo:affineInvariantMP} as $\gamma$ minimizing the RHS of the first equality coincides with the update of Algorithm~\ref{algo:affineInvariantMP} Line 5. The first equality holds as $\CfMPr$ is defined on $\rho \conv(\cA)$ and $\rho \conv(\cA)$ contains all iterates by definition, so that the unconstrained minimum lies in $[0,1]$. 

Using $\varepsilon_t = f(\bx^\star) - f(\bx_t)$, we can lower bound the error decay as 
\begin{eqnarray}\label{eqproof:linearMPAffineLowerBound}
\varepsilon_{t} - \varepsilon_{t+1} \geq \frac{\rho^2}{2\CfMPr} \left\langle \nabla f(\bx_t), \tilde{\bz}_t \right\rangle ^2.
\end{eqnarray}

Starting from the definition of $\mufr$ in~\eqref{def:muf}, we get,
\begin{eqnarray*}
\frac{\gamma(\bx_t, \bx^\star)^2}{2} \mufr &\leq & f(\bx^\star) - f(\bx_t) - \langle \nabla f(\bx_t), \bx^\star - \bx_t)\rangle \\ 
& = & -\varepsilon_t + \gamma(\bx_t, \bx^\star) \langle-\nabla f(\bx_t), \bs(\bx_t)\rangle ,
\end{eqnarray*}
which gives 
\begin{eqnarray}\nonumber
\varepsilon_t &\leq & - \frac{\gamma(\bx_t, \bx^\star)^2}{2} \mufr  + \gamma(\bx_t, \bx^\star)  \langle-\nabla f(\bx_t), \bs(\bx_t)\rangle  \\\nonumber
& \leq &  \frac{\langle-\nabla f(\bx_t), \bs(\bx_t)\rangle^2} {2 \mufr} \\ \label{eqproof:linearMPAffineUpperBound}
& = & \frac{\langle-\nabla f(\bx_t), \rho\tilde{\bz}_t\rangle^2} {2 \delta^2 \mufr} 
\end{eqnarray}
where the last inequality is by the quality of the approximate LMO as used in the algorithm, as defined in \eqref{eq:inexactLMOMP}.

Combining equations~\eqref{eqproof:linearMPAffineLowerBound} and \eqref{eqproof:linearMPAffineUpperBound}, we have
\begin{eqnarray*}
\varepsilon_{t} - \varepsilon_{t+1} \geq \delta^2 \frac{\mufr}{\CfMPr} \, \varepsilon_t ,
\end{eqnarray*}
which proves the claimed result. 
\end{proof}

Similar to the relationship between $C_f$ and smoothness, we explore the relationship of $\muf$ with strong convexity. Lemma~\ref{lemma:MUFwithMDW} is analogous to a similar result explored for the Frank-Wolfe case by~\cite{LacosteJulien:2015wj} relating an analogous quantity to the more complex notion of pyramidal width.

\begin{lemma}
\label{lemma:MUFwithMDW}
If $f$ is $\mu$-strongly convex over the domain $\rho \conv(\cA)$ with
respect to some arbitrary chosen norm $\norm{.}$ and $\mdw>0$, then 
\begin{equation}
\muf \geq \mu  \mdw^2
\end{equation}
\end{lemma}
\begin{proof}
By the definition of strong convexity, for any $\bx\in\rho\conv(\cA)$,
\begin{equation*}
D(\bx^\star, \bx) \geq \mu \| \bx^\star - \bx \|^2 .
\end{equation*}
Hence, 
\begin{eqnarray*}
\muf &\geq& \frac{\mu}{\gamma(\bx,\bx^\star)^2} \| \bx^\star - \bx\| \\
&= & \mu \frac{ \langle - \nabla f (\bx) , \bs(\bx) \rangle}{ \langle - \nabla f(\bx), \frac{\bx^\star - \bx}{\| \bx^\star - \bx \|}  \rangle }^2\\
\end{eqnarray*}

We now split $- \nabla f(\bx) = \bd_{\parallel} + \bd_{\perp}$, so that $\bd_\parallel \in \lin(\cA)$ while $\bd_\perp$ lies in the orthogonal complement of $\lin(\cA)$. 
Since $\bs(\bx), \bx^\star ,\bx$ all lie in $\rho\conv(\cA) \subseteq \lin(\cA)$, we get 
\begin{eqnarray*}
\muf &\geq & \mu  \frac{\langle \bd_\parallel, \bs(\bx)\rangle }{ \langle \bd_\parallel, \frac{\bx^\star - \bx}{\| \bx^\star - \bx \|}  \rangle }^2 \\
 & \geq & \mu  \frac{\langle \bd_\parallel, \bs(\bx)\rangle^2 } { \| \bd_\parallel \|^2} \\
 &=&  \mu  \Big\langle \frac{\bd_\parallel}{ \| \bd_\parallel \|}, \bs(\bx) \Big\rangle^2 \\ 
& \geq& \mu \mdw^2,
\end{eqnarray*} 
where the last inequality holds since $\mdw>0$. Indeed, it holds that:
\begin{eqnarray}
\langle \frac{\bd_\parallel}{ \| \bd_\parallel \|},\bs(\bx) \rangle =  \max_{\bz\in\cA} \langle \frac{-\bd_\parallel}{ \| \bd_\parallel \|},\bz\rangle \geq \mdw
\end{eqnarray}
which can be squared provided that $\langle \frac{\bd_\parallel}{ \| \bd_\parallel \|},\bs(\bx) \rangle$ and $\mdw$ are both positive which is the case if $\mdw>0$.
\end{proof}

\paragraph{Norm-based Variants}

\begin{reptheorem}{thm:linearRateMPinexact}
Let $\cA \subset \cH$ be a bounded  
set such that $\mdw>0$.
and let the objective function $f \colon \cH \to \R$ be both $L$-smooth and $\mu$-strongly convex over $\rho \conv(\cA)$.

Then, Matching Pursuit and Norm-Corrective Matching Pursuit (Algorithm~\ref{algo:generalgreedy}) converge linearly as
\begin{equation*}
\varepsilon_{t+1}
\leq \bigg(1- \delta^2\frac{\mu\mdw^2}{L\radius(\cA)^2} \bigg)\varepsilon_{t}
\end{equation*}
where $\varepsilon_t := f(\bx_t) - f(\bx^\star)$ is the suboptimality at step $t$, and $\delta \in (0,1]$ is the relative accuracy parameter of the employed approximate \lmo  (Equation~\eqref{eq:inexactLMOMP}).
\end{reptheorem}
\begin{proof}
From Lemma~\ref{lem:CfwithRadius}, $\CfMPr \leq L \rho^2 \radius_{\norm{.}}(\cA)^2$, while from Lemma~\ref{lemma:MUFwithMDW}, $\mufr \geq \mu \rho^2 \mdw^2$. Combining,   
\begin{eqnarray*}
\frac{\mufr}{\CfMPr} \geq \frac{\mu\mdw^2}{L\radius(\cA)^2} 
\end{eqnarray*}

The result now follows from Theorem~\ref{thm:LinearMPAffineInvariant}.
\end{proof}

\subsection{Proof of Corollary~\ref{cor:sublinleastsq}}

We have
\begin{equation}
\| \bx_t \|_\cA \leq \frac{\| \bx_t \|}{\inr(\conv(\cA))}, \label{eq:inrub}
\end{equation}
for all $t = 1\dots T$, and hence $\rho \leq \frac{\tilde \rho}{\inr(\conv(\cA))}$, which yields the first upper bound on $f(\bx_t) - f(\bx^\star)$ in Corollary~\ref{cor:sublinleastsq}. Here, \eqref{eq:inrub} essentially follows from the definition of the atomic norm. A formal proof can be obtained by writing the atomic norm as the value of an $\ell_1$-norm minimization problem \cite{Chandrasekaran2012} and using arguments from the proof of Theorem 2.5 in \cite{soltanolkotabi2012geometric}.

To get upper bound on $f(\bx_t) - f(\bx^\star)$ for $f(\bx) = \frac{1}{2}d^2(\bx,\by)$, note that for this particular choice of $f$, $\bx_t-\bb = \by$ and the update step in Algorithm \ref{algo:generalgreedy} can be written as $\bx_{t+1} = \cP_t \by$, where $\cP_t$ is the orthogonal projection operator onto $\cS$ in iteration $t$. Hence, we have $\| \bx_t \| \leq \| \by \|$, for all $t \in [T]$, as a consequence of $\| \cP_t \|_\mathrm{op} = 1$.

\subsection{Lower Bound}

\begin{reptheorem}{thm:lowerBoundLinearRateMPinexact}
Let $\cA \subset \cH$ be a bounded set and let the objective function $f \colon \cH \to \R$ be both $L$-smooth and $\mu$-strongly convex. Assume, $ \bx^\star := \argmin_{\bx\in \lin(\cA)} f(\bx) =\argmin_{\bx\in\cH} f(\bx) $.
Let $\varepsilon_t := f(\bx_t) - f(\bx^\star)$ be the suboptimality of the iterates and $\bz_t$ the atom selected at iteration $t$ by the \lmo. 
 Then, for $t\geq 0$ the suboptimality of the iterates of Matching Pursuit (Algorithm~\ref{algo:generalgreedy} variant 0) with an exact \lmo does not decay faster than:
\[
\varepsilon_{t+1}
\geq \left( 1 - \frac{W_\cA (-\nabla f(\bx_t))^2}{\|\bz_t\|^2} 
\frac{2L-\mu}{\mu}\right) \varepsilon_t
\]
\end{reptheorem}
\begin{proof}
First of all we note how the least-squares function is both $L$-smooth and $\mu$-strongly convex with $L=\mu$.
Recall that the new iterate $\bx_{t+1}$ is obtained as $\bx_{t+1}=\bx_t+\gamma\bz_t$ where $\gamma = - \frac{\langle \nabla f(\bx_t),\bz_t\rangle}{L\|\bz_t\|^2}$.
By strong convexity we get:
\begin{align*}
f(\bx_{t+1})&\geq f(\bx_t)-\frac{\langle \nabla f(\bx_t),\bz_t\rangle^2}{L\|\bz_t\|^2}+\\&+\frac{\mu }{2} \left( \frac{\langle \nabla f(\bx_t),\bz_t\rangle}{L\|\bz_t\|^2}\right)^2\|\bz_t\|^2
\end{align*}
and subtracting $f(\bx^\star)$ on both sides yields:
\begin{align}\label{step:lowerBoundStrongConv}
\varepsilon_{t+1}&\geq \varepsilon_t-\frac{\langle \nabla f(\bx_t),\bz_t\rangle^2}{L\|\bz_t\|^2}+\frac{\mu }{2} \left( \frac{\langle \nabla f(\bx_t),\bz_t\rangle}{L\|\bz_t\|^2}\right)^2\|\bz_t\|^2\nonumber\\
&\geq \varepsilon_t -\frac{1}{L} \Big\langle \nabla f(\bx_t),\frac{\bz_t}{\|\bz_t\|}\Big\rangle^2
\left(1-\frac{\mu}{2L}\right).
\end{align}
By smoothness we obtain:
\begin{align*}
&\ f(\bx_t+\gamma(\bx^\star-\bx_t))\\
\leq&\ f(\bx_t)+\gamma\langle \nabla f(\bx_t),\bx^\star-\bx_t\rangle + \gamma^2\frac{L}{2}\|\bx^\star-\bx_t\|^2.
\end{align*}
Now we can further lower bound the LHS by $f(\bx^\star)$ and minimize the RHS by $\gamma = - \frac{\langle \nabla f(\bx_t),\bx^\star-\bx_t\rangle}{L\|\bx^\star-\bx_t\|^2}$. Therefore:
\begin{align}
& f(\bx^\star) \leq \\
&\textstyle f(\bx_t)-\frac{1}{L}\langle \nabla f(\bx_t),\frac{\bx^\star-\bx_t}{\|\bx^\star-\bx_t\|}\rangle^2+\frac{1}{2L}\langle \nabla f(\bx_t),\frac{\bx^\star-\bx_t}{\|\bx^\star-\bx_t\|}\rangle^2\nonumber
\end{align}
which by definition of $\varepsilon_t$ becomes:
\begin{align}\label{step:lowerBoundLip}\textstyle
\varepsilon_t \geq \frac{1}{2L}\langle \nabla f(\bx_t),\frac{\bx^\star-\bx_t}{\|\bx^\star-\bx_t\|}\rangle^2 \notag\\\textstyle
= \frac{1}{2L}\|\nabla f(\bx_t)\|^2\langle \frac{\nabla f(\bx_t)}{\|\nabla f(\bx_t)\|},\frac{\bx^\star-\bx_t}{\|\bx^\star-\bx_t\|}\rangle^2 \\\textstyle
\geq \frac{1}{2L}\|\nabla f(\bx_t)\|^2 \frac{\mu}{L}.\notag
\end{align}
Recall the first-order optimality condition for constrained optimization. We have that $\langle \nabla f(\bx^\star), \bx^\star-\bx_t\rangle=0$. 
In the last inequality of Equation~\eqref{step:lowerBoundLip} we used the following arguments along with the first order optimality condition:
\begin{align*}
- \langle \nabla f(\bx_t), \bx^\star-\bx_t\rangle &\overset{\substack{\text{str. conv}\\+\\\text{opt. cond.}}}{
\geq}{\mu \|\bx^\star-\bx_t\|^2}
\end{align*}
multiplying and dividing by the norm of the gradient and rearranging we obtain:\vspace{-2mm}
\begin{align*}
- \left\langle \frac{\nabla f(\bx_t)}{\|\nabla f(\bx_t)\|}, \frac{\bx^\star-\bx_t}{\|\bx^\star-\bx_t\|}\right\rangle 
&\geq\mu \frac{\|\bx^\star-\bx_t\|}{\|\nabla f(\bx_t)\|}
\overset{\substack{\text{L-smooth}\\+\\\text{opt. cond.}}}{
\geq}\frac{\mu }{L}
\end{align*}
By taking the square we obtain the inequality used in Equation~\eqref{step:lowerBoundLip}.

Combining \eqref{step:lowerBoundStrongConv} with \eqref{step:lowerBoundLip} (the latter being $1\le 2L \tfrac L\mu \tfrac{1}{\|\nabla f(\bx_t)\|^2} \varepsilon_t$), we finally obtain:
\begin{align*}
\varepsilon_{t+1}
&\geq \varepsilon_t -\frac{\langle \nabla f(\bx_t),\frac{\bz_t}{\|\bz_t\|}\rangle^2}{\|\nabla f(\bx_t)\|^2}\frac{2L}{\mu} 
\left(1-\frac{\mu}{2L}\right) \varepsilon_t\\
&=\varepsilon_t -\frac{ W_\cA (-\nabla f(\bx_t))^2 }{\|\bz_t\|^2} \frac{2L}{\mu}
\left(\frac{2L - \mu}{2L}\right)\varepsilon_t
\\
&=\varepsilon_t - \frac{W_\cA (-\nabla f(\bx_t))^2}{\|\bz_t\|^2}\tfrac{1}{\mu}(2L-\mu) 
\varepsilon_t
\end{align*}
\end{proof}

\subsection{Proof of Theorem \ref{thm:coherenceVSdw}}\label{proof:coherenceVSdw}
Let $\cA\in\R^d$ be a symmetric set with $\|\bs\|_2 = 1$, for all $\bs \in \cA$, which spans $\R^d$. Let also $\cB$ be a set such that $\cA=\cB\cup-\cB$ with $\cB\cap-\cB=\emptyset$ and $|\cB|=n$.

Our proofs rely on the Gram matrix $\bG(\cJ)$ of the atoms in~$\cB$ indexed by $\cJ \subseteq [n]$, i.e., $(\bG(\cJ))_{i,j} := \langle \bs_i, \bs_j \rangle, i,j \in \cJ$.

To prove Theorem~\ref{thm:coherenceVSdw} 
, we use the following known results. 

\begin{lemma}[\citet{Tropp:2004gc}] The smallest eigenvalue $\gamma_{\min} (\bG(\cJ))$ of $\bG(\cJ)$ obeys $\gamma_{\min}(\bG(\cJ)) > 1 - \mu(\cB,m-1)$, where $m = |\cJ|$.
\label{thm:eigenval}
\end{lemma}

\begin{lemma}[\citet{devore08}] For every index set $\cJ \subseteq [n]$ and every linear combination $\bp$ of the atoms in $\cB$ indexed by $\cJ$, i.e., $\bp := \sum_{j \in \cJ} v_j \bs_j$, we have $\max_{j \in \cJ } | \langle \bp, \bs_j  \rangle | \ge \frac{\| \bp \|^2}{\|\bv\|_1} = \frac{\langle \bv, \bG(\cJ) \bv \rangle_2 }{\| \bv\|_1}$, where $\bv \neq \textbf{0}$ is the vector having the $v_j$ as entries.
\label{thm:lowerbounddotprod}
\end{lemma}

\begin{reptheorem}{thm:coherenceVSdw}
Let $\cA \subset \R^d$ be a symmetric set of $2n$ atoms with $\|\bs\|_2 = 1$ for all $\bs \in \cA$. Let also $\cB$ be a set such that $\cA=\cB\cup-\cB$ with $\cB\cap-\cB=\emptyset$ and $|\cB|=n$. Then,  the cumulative coherence of the set $\cB$ is bounded by: $\mu(\cB,n-1)\geq 1- n \cdot \mdw^2$.
\end{reptheorem}

\begin{proof}
For each direction $\bd\in\lin(\cA)$ with $\|\bd\|=1$ it holds that:
\begin{align*}
\max_{\bz\in\cA}|\langle \bd,\bz \rangle | &\overset{\text{Lemma~\ref{thm:lowerbounddotprod}}}{
\geq} {\frac{\|\bd\|^2}{\|\bv\|_1}}\\
&=\frac{\sqrt{\langle \bv,\bG(\cJ)\bv\rangle_2}\|\bd\|}{\|\bv\|_1}\\
&\geq \frac{\sqrt{\langle \bv,\bG(\cJ)\bv\rangle_2}}{\sqrt{n}\|\bv\|_2}\\
&\overset{\text{Lemma~\ref{thm:eigenval}}}{
\geq}  \sqrt{\frac{1-\mu(\cB,n-1)}{n}}
\end{align*}
This holds for every direction in $\lin(\cA)$, included the one that minimizes $\max_{\bz\in\cA}|\langle \bd,\bz \rangle | $. Therefore, we have:\begin{align*}
\mdw^2\geq\frac{1-\mu(\cB,n-1)}{n}
\end{align*}
Rearranging we obtain:
\begin{align*}
\mu(\cB,n-1)\geq 1- n \cdot \mdw^2.
\end{align*}
\end{proof}

\subsection{On the Relationship Between Matching Pursuit and Frank-Wolfe}\label{app:FWMP}
\begin{reptheorem}{thm:stepFWMPshort}
Let $\cA \subset \cH$ be a bounded set and let $f \colon \cH \to \R$ be a $L$-smooth convex function.
Let  $\alpha>0$ and let us fix an iteration $t>0$ and the iterate computed at the previous iteration $\bx_{t}$. If $- \frac{\langle\nabla f (\bx_{t}),\alpha\bz_t -\bx_{t}\rangle}{L\|\alpha\bz_t -\bx_{t}\|^2}\leq 1$ the new iterate $\bx_{t+1}^\FWa = \bx_{t} +\gamma(\alpha\bz_{t}-\bx_{t})\big|_{\gamma\in [0,1]}$ of Frank-Wolfe (Algorithm~\ref{algo:FW}) using the set $\alpha\cA=\left\lbrace \bx:\exists \bz\in\cA \ s.t. \  \bx=\alpha\bz\right\rbrace $ converges to the new iterate $\bx_{t+1}^\MP=\bx_{t}+\gamma\bz_t\big|_{\gamma\in \R}$ of Matching Pursuit (Algorithm~\ref{algo:generalgreedy} variant 0) applied on the linear span of the set $\cA$ with rate:
\begin{equation*}
 \|\bx_{t+1}^\FWa - \bx_{t+1}^\MP\| \in O\left(\frac{1}{\alpha}\right) 
\end{equation*}
In particular when $\alpha$ grows to infinity the condition $- \frac{\langle\nabla f (\bx_{t}),\alpha\bz_t -\bx_{t}\rangle}{L\|\alpha\bz_t -\bx_{t}\|^2}\leq 1$ always holds (for all steps $t$).
If the condition $- \frac{\langle\nabla f (\bx_{t}),\alpha\bz_t -\bx_{t}\rangle}{L\|\alpha\bz_t -\bx_{t}\|^2}\leq 1$ is not satisfied at step $t$
then the difference of the iterates increases linearly:
\begin{equation*}
 \|\bx_{t+1}^\FWa - \bx_{t+1}^\MP\| \in O\left({\alpha}\right) 
\end{equation*} 
\begin{proof}
In both the algorithms the new iterate is obtained by minimizing the quadratic approximation of $f$ computed at $\bx_{t}$.
Let $g_{\bx_{t}}(\bx)$ be the quadratic approximation of $f$ at $\bx_{t}$ given by
 \begin{equation*}
 g_{\bx_{t}}(\bx) = f(\bx_{t}) + \langle\nabla f(\bx_{t}), \bx-\bx_{t}\rangle+\frac{L}{2}\|\bx-\bx_{t}\|^2.
 \end{equation*} 
At iteration $t>0$ the new iterate $\bx_{t+1}^\FWa$  of Frank-Wolfe (Algorithm~\ref{algo:FW}) on the set $\conv(\alpha\cA)$ is computed as $\bx_{t+1}^\FWa = \bx_{t} + \gamma(\alpha\bz_t-\bx_{t})$, where:
\begin{align*}
\gamma &= \argmin_{\gamma\in[0,1]}g_{\bx_{t}}(\bx_{t} + \gamma( \alpha\bz_t - \bx_{t})) \nonumber \\
 &= f(\bx_{t}) + \nonumber\\&\quad\min_{\gamma\in[0,1]}\left\lbrace +\gamma  \langle\nabla f(\bx_{t}), \alpha\bz_t-\bx_{t}\rangle+\frac{L\gamma^2}{2}\|\alpha\bz_t-\bx_{t}\|^2\right\rbrace
\end{align*}
which solved for $\gamma$ yields:
\begin{equation}
\bx_{t+1}^\FWa = \begin{cases} \mbox{if } - \frac{\langle\nabla f (\bx_{t}),\alpha\bz_t -\bx_{t}\rangle}{L\|\alpha\bz_t -\bx_{t}\|^2}\leq 1: \\ \quad\bx_{t} - \frac{\langle\nabla f (\bx_{t}),\alpha\bz_t -\bx_{t}\rangle}{L\|\alpha\bz_t -\bx_{t}\|^2}(\alpha\bz_t-\bx_{t}) \\
 \mbox{otherwise:}\\
 \quad\alpha \bz_t  \end{cases}
\end{equation}
On the other hand, the new iterate $\bx_{t+1}^\MP$ of Matching Pursuit (Algorithm~\ref{algo:generalgreedy} variant 0) on the set $\cA$ is computed as $\bx_{t+1}^\MP=\bx_{t}+\gamma\bz_{t}$ where:
\begin{align*}
\gamma &= \argmin_{\gamma\in\R}g_{\bx_{t}}(\bx_{t} + \gamma \bz_t) \nonumber \\
 &= f(\bx_{t}) + \min_{\gamma\in\R}\left\lbrace +\gamma  \langle\nabla f(\bx_{t}), \bz_t\rangle+\frac{L\gamma^2}{2}\|\bz_t\|^2\right\rbrace
\end{align*}
which solved for $\gamma$ yields:
\begin{equation}
\bx_{t+1}^\MP = \bx_{t}-\frac{\langle \nabla f(\bx_{t}),\bz_t\rangle}{L\|\bz_t\|^2}\bz_t
\end{equation}
Now:
\begin{align*}
&\|\bx_{t+1}^\FWa - \bx_{t+1}^\MP\|=\\
 &\quad\begin{cases}  
 \mbox{if } - \frac{\langle\nabla f (\bx_{t}),\alpha\bz_t -\bx_{t}\rangle}{L\|\alpha\bz_t -\bx_{t}\|^2}\leq 1:\\
 \quad\|\frac{\langle \nabla f(\bx_{t}),\bz_t\rangle}{L\|\bz_t\|^2}\bz_t - \frac{\langle\nabla f (\bx_{t}),\alpha\bz_t -\bx_{t}\rangle}{L\|\alpha\bz_t -\bx_{t}\|^2}(\alpha\bz_t-\bx_{t})\| \\
 \mbox{otherwise:}\\
  \quad\|\alpha \bz_t+\frac{\langle \nabla f(\bx_{t}),\bz_t\rangle}{L\|\bz_t\|^2}\bz_t\| \end{cases}
\end{align*}
Using the fact that $\|\alpha\bz_t -\bx_{t}\|^2=\langle\alpha\bz_t -\bx_{t}, \alpha\bz_t -\bx_{t}\rangle$ it is easy to show that the distance depends on $1/\alpha$ when $- \frac{\langle\nabla f (\bx_{t}),\alpha\bz_t -\bx_{t}\rangle}{L\|\alpha\bz_t -\bx_{t}\|^2}\leq 1$ while is linear in $\alpha$ in the other case.
Furthermore, when $\alpha$ grows $- \frac{\langle\nabla f (\bx_{t}),\alpha\bz_t -\bx_{t}\rangle}{L\|\alpha\bz_t -\bx_{t}\|^2}\leq 1$ always holds. Therefore:
\begin{equation}
\lim_{\alpha\rightarrow \infty} \|\bx_{t+1}^\FWa - \bx_{t+1}^\MP\| = 0
\end{equation}
Therefore, the Frank-Wolfe step converges to the Matching-Pursuit step when $\alpha$ (i.e., the diameter of the set) grows to infinity.
\end{proof}
\end{reptheorem}

\subsection{Proof of Corollary \ref{Cor:lowerBoundEx}}\label{lowerBoundEx}
Let us discuss the following example.
Let $\Delta = \left\lbrace \be_1,\ldots,\be_d  \right\rbrace $ 
be the natural basis of $\R^d$ and assume we are minimizing $f(\bx)=\frac12 \|\by-\bx\|^2$ over the set $\lin(\cA) = \lin(\Delta \cup -\Delta) = \lin\left( \left\lbrace \be_1,-\be_1,\ldots,\be_d,-\be_d\right\rbrace\right)  $ 
(i.e., symmetrized $\Delta$) 
starting from $\bx_0 = \textbf{0}$.
We further assume that each component of the target vector $\by$ is equal to 1. This satisfies the assumptions of Theorem~\ref{thm:lowerBoundLinearRateMPinexact}. To estimate the lower bound constant we note that if $\mu=L$ then $\frac{L}{\mu}\frac{2L-\mu}{\mu}=1$. We now have to bound $\frac{W_\cA (-\nabla f(\bx_t))^2}{\|\bz_t\|^2}$. 
At iteration $t<d$ we know that $\nabla f(\bx_t) = -(\by - \bx_t)$. By the specific assumptions we made on the set $\cA$ and since $\by$ has a 1 in each component, $\bx_t$ has exactly $t$ values which are equal to $1$ and $d-t$ zeros. Note that the minimization over the span of the symmetrized natural basis ensures that Matching Pursuit and Norm-Corrective Matching Pursuit coincide by Gram-Schmidt. Indeed, this case is equivalent to computing the representation of the vector $\by$ in the natural basis, one component at the time. Each step affects only one of the components of the residual and keeps the rest untouched, which is to say that if they were zero at iteration $t-1$ they are zero also at iteration $t$. 
Therefore, the other $d-t$ non zero components have the same value they had at the first iteration, which is to say they are all equal to $-1$.
Then, we have
\begin{equation}\label{eq:exampleBounds}
\frac{W_\cA (-\nabla f(\bx_t))^2}{\|\bz_t\|^2}=\frac{1}{d-t}.
\end{equation}
The lower bound becomes:
\begin{align*}
\varepsilon_{t+1}\geq \left(1-\frac{1}{d-t}\right)\varepsilon_t. 
\end{align*}
Let us now consider the upper bound. Using the same argument of Equation~\eqref{eq:exampleBounds} we have that the exponential decay is:
\begin{align*}
\varepsilon_{t+1}\leq \left(1-\frac{1}{d-t}\right)\varepsilon_t.
\end{align*}
Therefore, the exponential decay is tight in this example. In Theorem~\ref{thm:linearRateMPinexact} we further bound the decay so that it depends only on the geometry of the atoms set. In this example we have $\mdw^2 = \frac{1}{d}$. Hence, we have:
\begin{align*}
\varepsilon_{t+1}  
\leq \left(1-\frac{1}{d-t}\right)\varepsilon_t
\leq \left(1-\frac{1}{d}\right)\varepsilon_t.
\end{align*}

Therefore, the ratio $\frac{\mdw}{\radius(\cA)}$ in the linear convergence rate (Theorem~\ref{thm:linearRateMPinexact}) makes it loose in this example. On the other hand, it does not depend on the particular iterate and yields the global linear convergence rate. Note that the $\mdw$ is a geometric quantity and does not depend on either $\bx_0$ or~$\by$. On the other hand, in the case of the symmetrized natural basis we can give an explicit value for $W_\cA (-\nabla f(\bx_t))$ at every iteration. 

After this discussion, the proof of Corollary~\ref{cor:sublinleastsq} is trivial considering that at iteration $t$ the gradient changed from the first iteration only in the indexes in $\cI$ (they became zero).

\end{document}